\documentclass[final,12pt]{colt2022} % Anonymized submission
\usepackage{mathtools}
\usepackage{bm}
\usepackage{dsfont}
\usepackage{comment}
\usepackage{makecell}
\usepackage{hhline}
\usepackage{color,soul}
\usepackage{multirow}
\usepackage{tablefootnote}

\DeclareBoldMathCommand{\one}{1}

\newcommand{\R}{\mathbb{R}}
\newcommand{\reals}{\mathbb{R}}

\newcommand{\NN}{\mathbb{N}}

\renewcommand{\epsilon}{\varepsilon}

\newcommand{\eps}{\varepsilon}

\newcommand{\NORM}[1]{\lVert #1 \lVert}
\newcommand{\actions}{\mathcal{A}}
\newcommand{\IP}[2]{\langle #1 , #2 \rangle}
\newcommand{\domX}{\mathcal{X}}
\newcommand{\domY}{\mathcal{Y}}
\newcommand{\sample}{\mathcal{S}}

\newcommand{\funclass}{\mathcal{F}}

\newcommand{\ball}{\mathcal{B}}
\newcommand{\loss}{\ell}
\newcommand{\logloss}{\ell_\textnormal{log}}

\newcommand{\der}{\mathrm{d}}
\newcommand{\intder}{\,\der}

\newcommand{\union}{\cup}

\newcommand{\bmax}{\vee}   % binary maximum
\newcommand{\bmin}{\wedge} % binary minimum
\newcommand{\pnml}{p_\textnormal{nml}} % normalised maximum-likelihood density
\newcommand{\mlmu}{\hat{\mu}}       % sample mean
\newcommand{\mltheta}{\hat{\theta}} % maximum likelihood parameter
\newcommand{\normaldist}{\mathcal{N}} % normal distribution
\newcommand{\bernoullidist}{\mathcal{B}} % Bernoulli distribution
\newcommand{\alphamax}{\alpha_\textnormal{max}} % maximum value for alpha
\newcommand{\alphamin}{\alpha_\textnormal{min}} % minimum value for alpha
\newcommand{\Xmin}{X_\textnormal{min}} % minimum value for X
\newcommand{\ind}[1]{1_{[#1]}} % indicator

\newcommand{\ceil}[1]{\left\lceil #1 \right\rceil}
\newcommand{\floor}[1]{\left\lfloor #1 \right\rfloor}

\renewcommand{\leq}{\leqslant}
\renewcommand{\geq}{\geqslant}

\DeclareMathOperator*{\E}{\mathbb{E}}
\DeclareMathOperator{\Ber}{Ber}

\DeclareMathOperator{\Vol}{Vol}
\DeclareMathOperator{\Tr}{Tr}

\DeclareMathOperator{\KL}{KL}
\DeclareMathOperator{\argmin}{argmin}
\DeclareMathOperator{\argmax}{argmax}

\DeclareMathOperator{\sign}{sign}

\DeclareMathOperator{\Proj}{Proj}

\newcommand{\algname}[1]{\rmfamily{#1}}

\newtheorem*{theorem*}{Theorem}
\newtheorem*{lemma*}{Lemma}

\title{Scale-free Unconstrained Online Learning for Curved Losses}

\usepackage{times}
\coltauthor{%
 \Name{Jack J. Mayo}\normalfont{,} \Email{j.j.mayo@uva.nl}\\
 \Name{H\'edi Hadiji} \Email{h.hadiji@uva.nl}\\
  \normalfont{and} \Name{Tim {van~Erven}} \Email{tim@timvanerven.nl}\\
 \addr Korteweg-de Vries Institute for Mathematics, University of Amsterdam, Amsterdam, The Netherlands%
}

\begin{document}

\maketitle

\begin{abstract}
A sequence of works in unconstrained online convex optimisation have
investigated the possibility of adapting simultaneously to the norm $U$
of the comparator and the maximum norm $G$ of the gradients. In full
generality, matching upper and lower bounds are known which show that
this comes at the unavoidable cost of an additive $G U^3$, which is not
needed when either $G$ or $U$ is known in advance. Surprisingly, recent
results by \citet{kempka2019adaptive} show that no such price for
adaptivity is needed in the specific case of $1$-Lipschitz losses like
the hinge loss. We follow up on this observation by showing that there
is in fact never a price to pay for adaptivity if we specialise to any
of the other common supervised online learning losses: our results cover
log loss, (linear and non-parametric) logistic regression, square loss
prediction, and (linear and non-parametric) least-squares regression. We also fill in several gaps in the literature
by providing matching lower bounds with an explicit dependence on $U$.
In all cases we obtain scale-free algorithms, which are suitably
invariant under rescaling of the data. Our general goal is to
establish achievable rates without concern for computational efficiency,
but for linear logistic regression we also provide an adaptive method
that is as efficient as the recent non-adaptive algorithm by
\citet{AgarwalKaleZimmert2021}.
% Our main technical tools are mixability, which enables aggregation
% over multiple hyperparameter settings, and clipping of response
% variables.
\end{abstract}


\section{Introduction}

The problem of hyperparameter tuning is ubiquitous across machine
learning. We study it in the context of \emph{online supervised
learning} (see e.g.\ \citep{RakhlinSridharanTewari2015}), in which a
learner needs to issue sequential predictions over the course of $T$
rounds. At the start of each round $t$, the learner first receives a
feature vector $x_t$, and then issues a prediction $a_t$ of the
corresponding response $y_t$. Performance is measured by the
\emph{regret}, which is the difference between the sum of the learner's
losses $\loss(a_t,y_t)$ and the sum of the losses
$\loss(f_\theta(x_t),y_t)$ suffered by the best comparator function
$f_\theta$ from a function class $\funclass$ indexed by parameters
$\theta$.

There are two main types of hyperparameters: on the one hand it is
desirable to adapt automatically to the norm $\|\theta\|$ of the
optimal comparator parameters; on the other hand we want algorithms that
do not have to know the scale of the data $x_t$ and $y_t$ beforehand.
These issues have frequently been studied in the context of \emph{online
convex optimisation} (OCO) \citep{HazanOCO}, where underlying
details of the setup are abstracted away by assuming only that the
functions $\loss_t(\theta) := \loss(f_\theta(x_t),y_t)$ are convex and
requiring predictions to be of the form $a_t = f_{\theta_t}(x_t)$ for
some $\theta_t$. The scale of the data then comes in through the maximum
length $G$ of the gradients $g_t := \nabla \loss_t(\theta_t)$. Given $G$
and an upper bound $U \geq \|\theta\|$ for the optimal parameters
$\theta$, the best regret that can be guaranteed is $O(UG\sqrt{T})$
\citep{Zinkevich03}. This is even possible when only $U$ is known, but
not $G$ \citep{dutchietal}. Conversely, given $G$ but not $U$, it has
been found that the optimal rate is
$O(\|\theta\|G\sqrt{T\log(1+\|\theta\|)T})$
\cite{mcmahanstreeter,mcmahanabernethy,pmlr-v75-cutkosky18a}, so then
the price of adaptivity is a mere logarithmic factor in $\|\theta\|$.
Simultaneous adaptivity to both $G$ and $\|\theta\|$, however, has been
shown to be impossible without a worse dependence on $\|\theta\|$
\citep{Cutkosky2017} and comes at the non-negligible cost of an additive
$G \|\theta\|^3$, with matching upper and lower bounds establishing the
rate to be $O(\|\theta\| G\sqrt{T\log(1+\|\theta\|)}) + G \|\theta\|^3)$
\citep{cutkosky2019artificial,mhammedi2020lipschitz}. Alternatively,
$O((\|\theta\|^2 + 1)G\sqrt{T})$ is also possible
\citep{orabona2018scale}. This fully settles the issue of simultaneous
adaptivity, but only for the OCO setting.

Since we have more information available in online supervised learning,
the OCO lower bounds do not apply, and indeed
\citet{kempka2019adaptive,mhammedi2020lipschitz} obtain upper bounds of
order $O(U X \sqrt{T \log(1+UX T)})$ for linear models
$f_\theta(x_t) = \theta^\top x_t$ and losses of the form
$\loss_t(\theta) = h_t(\theta^\top x_t)$, where $h_t$ is $1$-Lipschitz
and $X = \max_{t\leq T} \|x_t\|_2$. $1$-Lipschitzness is satisfied by
important practical cases like the hinge loss $h_t(z) = \max\{1-y_t
z,0\}$, the two-class logistic loss $h_t(z) = \ln(1+e^{-y_t z})$ and the
absolute loss $h_t(z) = |y_t - z|$. The key feature of this bound is
that it depends on $U$ and $X$ only via their product $UX$, without
having to know either hyperparameter in advance. It is therefore both
adaptive to the norm of the comparator and scale-free: if all $x_t$ get
scaled by the same constant, then the optimal parameters $\theta$ undo
this scaling, and the bound remains unchanged. In fact, even the
algorithms are scale-free: scaling all $x_t$ does not affect the
predictions $a_t$ at all. This tantalising possibility of circumventing
lower bounds prompts us to ask the following general question:
\begin{quote}
  \textit{Given a specific loss $\loss$ and function class $\funclass$ in online
  supervised learning, what is the price of adapting to $\|\theta\|$
  while being at the same time scale-free?}
\end{quote}
We focus on answering this question for two major classes of losses
$\loss$: the first is the logarithmic loss $\loss(p,y) = -\ln p(y)$
where predictions $a \equiv p$ are densities or probabilities, with
(multiclass) logistic regression as its main special case; the second is
the square loss $\ell(a,y) = \|y-a\|^2$, which pertains to least-squares
regression. In the latter case, scale-freeness also requires the
predictions $a$ to scale linearly with the $y_t$, and the bounds also
depend on $Y = \max_{t \leq T} \|y_t\|$. Prior work and our
contributions are summarised in Table~\ref{tab:summary}.

The main observation from Table~\ref{tab:summary} is that there is \emph{never}
a price to pay in the rates for adapting to $\|\theta\|$ with a
scale-free algorithm, except possibly in the case that we do not study
here: for the hinge loss there exists a gap between known upper and
lower bounds in the regime where $\|\theta\| X > 1$, which we leave as
an open issue.

\paragraph{Approach}

Our main technical tool in obtaining upper bounds is mixability of the
logarithmic and square loss, which implies that we can aggregate over an
exponentially spaced grid of hyperparameters $\alpha \geq \alphamin > 0$
at the cost of a mere additive $O(\frac{1}{\eta} \log \log
(\alpha/\alphamin))$ term in the bound. This may be interpreted as the
number of bits to encode $\alpha$ rounded up to the nearest grid point.
A technical complication that requires considerable care is to specify a
minimum value $\alphamin$ without breaking either scale-freeness or
paying a non-negligible price in the bound. This is related to the
range-ratio problem of \citet{mhammedi2020lipschitz}. As a consequence,
we do end up with a dependence on the feature vector ratio $X_T/X_{t^*}$
in some cases, where $X_t = \max_{s \leq t} \|x_s\|$ and $t^*$ is the
smallest $t$ for which $\|x_t\|>0$. A logarithmic dependence on this
ratio has previously been considered acceptable by
\citet{kempka2019adaptive,ross2013normalized,wintenberger2017optimal,kotlowski2017scale}. In our case the term
appears inside an even smaller double logarithm, which means that it can
be neglected simply based on the range of numbers representable on a
computer as double precision floating point numbers. A similar doubly
logarithmic dependence was encountered by \citet{gerchinovitz11a}.
Mixability further depends on a parameter $\eta$, which is $\eta = 1$
for log loss and $\eta \propto 1/Y^2$ for square loss. In case of the
square loss, the fact that $Y^2$ is unknown in advance introduces the
need for online clipping and projecting of predictions to the range
$Y_t$, where $Y_t = \max_{s\leq t} \|y_s\|$. Similar approaches have
previously been used by \citet{gerchinovitz11a,cutkosky2019artificial}.

\newcommand{\mystrut}[1]{\rule{0mm}{#1}}
\begin{table}[h!]\label{tab:summary}
\scriptsize
\centering
\begin{tabular}{|c|c|c|c|}
\hline
\mystrut{2.5ex}\footnotesize \textbf{Loss} & \footnotesize \textbf{Function Class} & \footnotesize \textbf{Non-adaptive Rate}   &
\footnotesize \textbf{Adaptive Rate} \\
\hhline{|=|=|=|=|} 
\makecell{Logarithmic\\
$-\ln p(y)$}  & normal location & \makecell{\rule{0mm}{2.5ex} $
d\ln\frac{UT}{\sigma}$ \\ \citep{BarronRissanenYu1998}\\\citep{StineFoster2000}, Thm.~\ref{thm:normallocation} }   &  \makecell{\rule{0mm}{2.5ex} $d\ln\frac{\NORM{\theta}T}{\sigma}$  \\  \citep{grunwald2007minimum}, Thm.~\ref{adaptiveNormallocation}}  \\ 
\hline \multirow{3}{*}{ 
\makecell{Multiclass logistic \\ regression\\ (K classes)}}  &
linear &\makecell{\rule{0mm}{2.7ex}$\leq dK\ln\frac{UXT}{dK}$  \\
$UX = \Omega(\sqrt{d}\ln T)$:\; $\geq d\ln\frac{UX}{\sqrt{d}\ln T}$ \\
\citep{foster2018logistic}\\
$UX \leq 2\sqrt{d}$ :\;  $\geq d \ln \frac{UX T}{d}$\\
Thm.~\ref{thm:lower_bound_logreg}} & \makecell{ \rule{0mm}{2.7ex}
$\leq X\NORM{\theta} \sqrt{{T\ln(X\NORM{\theta} T) }}$
\\  \citep{mhammedi2020lipschitz}\\
$\leq dK\ln\frac{{\|\theta\|}XT}{dK}$ \\  Thm.~\ref{adaptivelogisticregression}, Thm.~\ref{thm:adaptivelogistic_paramfree}} \\ \cline{2-4}
& linear (efficient alg) & \makecell{\rule{0mm}{2.5ex}$\leq (UX+\ln K)dK\ln T$ \\  \citep{AgarwalKaleZimmert2021} }& \makecell{ \rule{0mm}{2.5ex}$\leq \left(\NORM{\theta}X+\ln K \right)dK\ln T$ \\  Thm.~\ref{thm:logistic_eff_ad_sf}} \\
\cline{2-4}
& Besov & \makecell{\mystrut{2.5ex} $\leq \tilde{O}(U^\beta T^\gamma)$
\citep{foster2018logistic}} & \makecell{\mystrut{2.5ex}
$\leq \tilde{O}(\|\theta\|^\beta T^\gamma)$ Thm.~\ref{thm:logistic_Besov}}\\
\hline
\makecell{Square \\ $\frac{1}{2}(y - a)^{2}$} & square loss prediction &
\makecell{\rule{0mm}{2.7ex}$Y^{2}\ln \frac{(U\bmin Y)T}{Y}$ \\
\citep{van-der-hoeven2018the-many}, Thm.~\ref{thm:sq_lower_bound}} &
\makecell{\rule{0mm}{2.7ex}$Y^{2}\ln\frac{(\NORM{\theta}\bmin Y)T}{Y}$ \\  Thm.~\ref{thm:sq_loss_full_bound}}\\
\hline
\multirow{2}{*}{\makecell{Least-squares\\regression }} & linear & 
\makecell{\rule{0mm}{2.5ex} $dY^{2}\ln\frac{UXT}{dY}$
\\  \citep{Vovk98competitiveregression,AzouryWarmuthRelativelossODE}\\
Thm.~\ref{thm:sq_lower_bound}}&
\makecell{$dY^{2}\ln\frac{ \NORM{\theta}X T}{dY}$ 
\\   Thm.~\ref{thm:regression-finite-dim}}\\
\cline{2-4}
& Sobolev, $s \geq d/2$ 
& \makecell{ \rule{0mm}{2.7ex} $ \tilde O(T^{d / (2s  + d)})$ \\ \citep{zadorozhnyi2021online}}
& \makecell{\rule{0mm}{2.7ex} $\leq \tilde O(\| \theta \|^{s / (2s + d)} T^{d / (2s  + d)})$ \\ Thm.~\ref{thm:kaar-sobolev}}
\\
\hline
\makecell{Hinge \\ $\max\ \! \{0, 1-ya \}$ } 
& linear 
& \makecell{\rule{0mm}{3ex}$\leq UX\sqrt{T}$\\
$\geq(UX\bmin 1)
\sqrt{T}$\tablefootnote{\label{foot:tolinear}By
reduction to linear loss, which works only when $UX \leq 1$; see Appendix~\ref{app:hinge} for details}}
& \makecell{$\rule{0mm}{3ex} \leq \NORM{\theta}X
\sqrt{{T\ln(\NORM{\theta}X T) }}$ \\  \citep{mhammedi2020lipschitz}\\
$\geq (\|\theta\|X\bmin 1) \sqrt{T\ln((\|\theta\|X \bmin
1)T)}$\textsuperscript{\ref{foot:tolinear}} \\
\citep{mcmahanstreeter}
} \\
\hline
\end{tabular}
\caption{Comparison of non-adaptive and adaptive rates for frequently
used losses. All adaptive rates are achieved by scale-free algorithms,
with no prior knowledge about the data.
} 

\end{table}
%\vspace{-1cm}
\paragraph{Types of Scale-freeness}

Finally, we remark that the appropriate definition of `scale-free'
depends on the loss and setting. In OCO, the focus has been on
algorithms whose predictions $\theta_t$ are invariant under scaling of
the gradients $g_t$. Since this does not imply that $\theta_t^\top x_t$ is
invariant under rescaling of $x_t$,
\citet{kempka2019adaptive,mhammedi2020lipschitz} add a post-processing
step to scale $\theta_t$ to the range of $1/X_t$. While we consider only
scale-freeness with respect to lengths of the whole vectors $x_t$,
refined invariances with respect to the scale of individual features
\citep{kempka2019adaptive, orabona2015generalized} or rotations
\citep{mhammedi2020lipschitz} have also been studied.
For the square loss, we also consider scale-freeness with respect to 
the data $y_t$. In non-parametric regression, the range of $x_t$ is
always assumed known, so we do not need to adapt to it.

\paragraph{Outline}

After preliminary definitions, we first study the logarithmic loss and
logistic regression in Section~\ref{sec:loglossandregression}. Then we
consider the square loss and least-squares regression in Section~\ref{sec:slls}.

\paragraph{Setting and Preliminaries}

We consider supervised online learning, in which the learner needs to
issue a prediction $a_t \in \actions$ for $y_t \in \domY$ at time $t$
based on all the previous observations $\sample_{t-1} =
(x_1,y_1),\ldots,(x_{t-1},y_{t-1}) \in (\domX \times \domY)^{t-1}$ as
well as the features $x_t \in \domX$. Given a loss function $\loss :
\actions \times \domY \to \reals \union \{\infty\}$, the performance of the learner after
$T$ time steps relative to a class of functions $\funclass = \{f_\theta :
\domX \to \actions \mid \theta \in \Theta\}$ is evaluated by the regret
\[
  R_T(\theta) =
  \sum_{t=1}^{T}\ell(a_{t},y_{t})
    - \sum_{t=1}^{T}\ell(f_\theta(x_{t}),y_{t})
  \qquad
  \text{for $\theta \in \Theta$.}
\]
For logistic and square loss, define the running maximum of the
feature norms $X_t = \max_{s \leq t} \|x_s\|$ and the responses
$Y_t = \max_{s \leq t} \|y_s\|$, where the relevant norm will be clear
from context.

\section{Logarithmic Loss and Logistic Regression} \label{sec:loglossandregression}

For the log(arithmic) loss, the set of allowed predictions $\actions$
corresponds to all probability density functions over $\domY$ with
respect to some common $\sigma$-finite measure $\nu$. Given a density
$p \in \actions$ and observation $y \in \domY$, the log loss is 
$
  \logloss(p,y) = -\ln p(y).
$

To emphasize that predictions are densities (or probability mass
functions if $\nu$ is the counting measure), we will write $p_t$ instead
of $a_t$ and $p_{\theta,t}$ for $f_\theta(x_t)$. We will consider the
log loss with respect to the normal location family and with respect to
the multiclass logistic regression probability model, building our
results on the Bayesian prediction strategy in both cases. Given a prior
distribution $\pi$ on $\Theta$, the Bayesian prediction strategy
predicts according to
\begin{equation}\label{eqn:logloss_substitution}
  p_t(y) = \int p_{\theta,t}(y) \intder \pi(\theta \mid \sample_{t-1}),
\quad \text{where } \quad 
  \der \pi(\theta \mid \sample_{t-1}) =
    \frac{\prod_{s=1}^{t-1} p_{\theta,s}(y_s) \intder \pi(\theta)}
         {\int \prod_{s=1}^{t-1} p_{\theta',s}(y_s) \intder
         \pi(\theta')},
\end{equation}
for which we assume throughout that the denominator is non-zero and
finite. We also note that these definitions presume that the map
$(\theta,y) \mapsto p_{\theta,t}(y)$ is measurable. 
% %
% \begin{lemma}\label{lem:logloss}
%   The Bayesian prediction strategy with prior $\pi$ achieves
%   \begin{equation}
%     \sum_{t=1}^{T} \logloss(p_t,y_t) \leq \E_{\theta \sim
%     \gamma}\bigg[\sum_{t=1}^T \logloss(p_{\theta,t},y_t)\bigg]+\KL(\gamma\| \pi)
%     \qquad
%     \text{for all distributions $\gamma$,}
%   \end{equation}
%   %
%   with equality if $\gamma = \pi(\theta \mid \sample_T)$.
% \end{lemma}
%

\paragraph{Adapting to a Hyperparameter}
The Bayesian prediction strategy can be applied directly to the normal
location family or the logistic loss, but it may also be used to
aggregate a finite or countable number of experts indexed by $\theta \in
\Theta \subset \{0,1,2,\ldots\}$, whose predictions $p_{\theta,t}$ may
vary arbitrarily over time. This fits into the general setting by
letting $\domX = \actions^\Theta$ and $f_\theta(x) = x(\theta)$, with
the interpretation that $p_{\theta,t} = f_\theta(x_t) = x_t(\theta)$ is
the prediction of expert $\theta$ at time $t$. By making each expert
correspond to a specific setting of a hyperparameter $\alpha$, it then
follows from Lemma~\ref{lem:logloss} (Appendix~\ref{app:logisticloss}) that we can adapt to $\alpha$ with
an overhead that is of order $O(\log \log \alpha)$:
\begin{lemma}
\label{lem:aggregate}
Suppose that $A(\alpha)$ is an algorithm for the log loss that depends on 
hyperparameter $\alpha \geq 0$ and achieves a regret bound
$B_T(\theta,\alpha) \geq R_T(\theta)$ for any $\theta \in
\Theta_{\alpha}$, where $\Theta_\alpha \subseteq \Theta_\beta$ for
$\alpha \leq \beta$. Then, for any $0 < \alphamin < \alphamax \leq
\infty$, it is possible to adapt to $\alpha \leq \alphamax$ with regret
bounded by
\begin{equation}\label{eqn:aggregate}
  R_T(\theta)
    < \max_{\alpha' \in [\alpha,2\alpha \bmax \alphamin]} B_T(\theta,\alpha')
    + 2\ln \log_2\Big(\frac{8\alpha}{\alphamin} \bmax 4\Big)
  \quad
  \text{for all $\alpha \in [0,\alphamax]$ and $\theta \in \Theta_{\alpha}$}
\end{equation}
by aggregating experts $A(\alpha)$ for $\alpha$ in the exponential grid
$\{\alphamin 2^m \mid m = 0,1,\ldots,M\}$ using the Bayesian prediction
strategy with prior $\pi(m) = \frac{M+2}{(M+1)(m + 1)(m+2)}$, where $M =
\ceil{\log_2(\alphamax/\alphamin)}$.
\end{lemma}
In Appendix~\ref{app:logtomix}, we recall how the Bayesian prediction strategy generalises to the Exponential Weights (EW) algorithm when replacing the log loss by \emph{mixable} losses \citep{Vovk01CompetitiveOS}. We apply this to the central case of the square loss in Section~\ref{sec:slls}.

\subsection{Warm-up: Normal Location Family}\label{sec:normallocationfam}

We start with the normal location family, which is simple
because there are no features (i.e.\ $\domX$ is a singleton), so
scale-freeness is not an issue, and we can study comparator-adaptivity
by itself. It also has the advantage that the non-adaptive minimax
regret can be calculated in closed form. In this case $\nu$ is the
Lebesgue measure, $\domY = \reals^d$, and
\[
    p_{\theta,t}(y) = \frac{\exp\big(-\|y -
    \theta\|_2^2 / (2\sigma^2)\big)}{(2 \pi \sigma^2)^{d/2}},
\]
with a fixed, known choice $\sigma > 0$ and $\Theta \subseteq \reals^d$.
We start with the exact minimax regret in the non-adaptive case, when
$\theta$ is constrained to a ball of known radius $U$, and then apply
Lemma~\ref{lem:aggregate} to adapt to $U$.
As observed by \citet{StineFoster2000,BarronRissanenYu1998}, the minimax
regret for $d=1$ can be computed exactly when $\Theta = \{\theta :
|\theta| \leq U\}$. The generalisation of their approach to higher
dimensions gives the following:
\begin{theorem}[Non-adaptive Minimax Rate]\label{thm:normallocation}
For any $U > 0$, the minimax regret for the log loss with respect to the
normal location family with $\Theta = \ball(0,U)$ equals
\begin{equation}
  \min_{\mathbf{Algs}} \max_\sample \max_{\theta \in \ball(0,U)} R_T(\theta)
    = \frac{d}{2} \ln \frac{T
    U^2/\sigma^2}{2\Gamma(\frac{d}{2}+1)^{2/d}} + V(U, T) \, ,
      \label{eqn:logminmax}
\end{equation}
where $V(U,T) = \ln \Big(1 + \frac{d}{\sqrt{T}U/\sigma} \int_0^\infty
\Big( 1+\frac{r}{\sqrt{T} U/\sigma}\Big)^{d-1} e^{-r^2/2}
\intder r\Big) = O(1/\sqrt{T})$.
\end{theorem}
To interpret these expressions, note that $\Gamma(\frac{d}{2}+1)^{2/d}
\approx \frac{d}{2e}$ by Stirling's approximation; in addition, $V(U,T)$ is a
lower-order term, which simplifies to $V(U,T) = \ln(1 +
\frac{\sqrt{\pi}}{\sqrt{2T}U/\sigma})$ for $d=1$. We see here that the
dependence on $U$ is only logarithmic, rather than linear, which turns
out to be common for curved losses when combined with parametric models. Adapting to $U$ in Theorem~\ref{thm:normallocation} using
Lemma~\ref{lem:aggregate} gives the following adaptive result, which may
be viewed as the straightforward generalisation of \citet[Examples~11.1
and 11.5]{grunwald2007minimum} to $d > 1$:
\begin{theorem}[Adaptive Rate]\label{adaptiveNormallocation}
  There exists a learner whose regret for the log loss with respect to
  the normal location family with $\Theta = \reals^d$ is at most
  \begin{equation*}
    R_T(\theta)
      \leq 
      \frac{d}{2} \ln \frac{2 T
      \|\theta\|^2/\sigma^2 + 1}{2\Gamma(\frac{d}{2}+1)^{2/d}} 
      + 2\ln \log_2\Big(\frac{8 T \|\theta\|^2}{\sigma^2} +
      4\Big)
      + V(\|\theta\|,T)
      = O\!\left(\frac{d}{2} \ln \frac{T\|\theta\|^2}{\sigma^2}\right),
  \end{equation*}
  for all $\theta \in \reals^d$, where $V(U,T) = O(1/\sqrt{T})$ is as in
  Theorem~\ref{thm:normallocation}.
\end{theorem}
Comparing to Theorem~\ref{thm:normallocation}, we see that the overhead
for adaptivity is negligible compared to the non-adaptive rate. Adaptivity to $T$ 
can be obtained by another application of Lemma~\ref{lem:aggregate} 
with $T = \alpha$.

\begin{proof}
  We apply Lemma~\ref{lem:aggregate} with $U = \sqrt{\alpha} =
  \|\theta\|$, $\alphamin=\sigma^2/T$ and $\alphamax = \infty$. The
  result then follows upon observing that $V(\sqrt{\alpha},T)$ is
  decreasing in $\alpha$.
\end{proof}

%\subsubsection{Bounded Observations}
%
%Leave this late:
%\begin{itemize}
%\item Approximate minimax regret for normal location family with known
%$U$ in general dimensions with bounded observations, for $\Sigma=I$.
%(Upper and lower bounds.)
%\item Can adapt to $U$ using mixability lemma
%\end{itemize}

\subsection{Multiclass Logistic Regression}\label{sec:logreg}

We proceed with multiclass logistic regression, which corresponds to the
case where $\nu$ is the counting measure on $K$ classes $\domY =
\{1,\ldots,K\}$, and the corresponding probability mass functions are
\[
    p_{\theta,t}(y) = \frac{e^{h_\theta(x_t)_y}}{\sum_{y' \in \domY}
    e^{h_\theta(x_t)_{y'}}},
\]
where $h_\theta : \domX \to \reals^K$ are predictors that map inputs $x
\in \domX \subseteq \reals^d$ to vectors of class-scores. In particular,
linear predictors $h_\theta(x) = \theta x$ are parameterised by weight
matrices $\theta \in \Theta \subseteq \reals^{K \times d}$ and lead to
the multiclass logistic loss $\logloss(p_{\theta,t},y) = \ln(1 +
\sum_{y' \neq y} e^{(\theta x_t)_{y'}-(\theta x_t)_y})$ when combined
with the log loss. The standard definition for binary logistic
regression with a single vector $\theta' \in \reals^d$ is recovered by
setting $\Theta \subset \{\big(\begin{smallmatrix} \theta' \\ 0
\end{smallmatrix}\big) : \theta' \in \reals^d\}$. We call an algorithm
for the logistic loss \emph{scale-free} if scaling all $x_t$ by the same
positive constant does not change the predictions $p_t$.

We discuss linear predictors, both in terms of minimax rates and for the rates that are achievable by efficient algorithms. In Appendix~\ref{app:besov}, we also consider adapting to the Besov norm of functions in non-parametric logistic regression.
%an application of the aggregation procedure trivially recovers adaptive versions of the nonconstructive bounds for Besov spaces achieved by \citet{foster2018logistic} (Theorem~\ref{thm:logistic_Besov}).

\subsubsection{Linear Predictors}

For linear predictors, there is a gap between the minimax rate and
the best known upper bound for which there exists an efficient
algorithm. We discuss the two cases in turn.

Let $\|\cdot\|$ be any norm, with corresponding dual norm $\|\cdot\|_*$,
and define the following induced matrix norm: $\|\theta\| = \max_k
\|\theta_k\|$, where $\theta_k$ is the $k$-th row of $\theta \in
\reals^{K \times d}$. Then \citet{foster2018logistic} provide the
following upper and lower bounds for the minimax rate:
\begin{theorem}[Non-adaptive Upper Bound,
\citet{foster2018logistic}]\label{thm:logisticupper}
  Suppose $\Theta \subseteq \{\theta \in \reals^{K \times d} :
  \|\theta\| \leq U\}$ is a non-empty convex set, and $\domX = \{x \in
  \reals^d : \|x\|_* \leq X\}$. Then the Bayesian prediction strategy
  with uniform prior $\pi$ on $\Theta$ satisfies
  \[
    R_T(\theta)
      \leq 5 d_\Theta \ln \Big(\frac{UXT}{d_\Theta} + e\Big)
      \qquad \text{for all $\theta \in \Theta$,}
  \]
  where $d_\Theta \leq d K$ is the linear-algebraic dimension of $\Theta$.
\end{theorem}
Indeed, the proof of this result by \citeauthor{foster2018logistic} may
be viewed as a specialisation of Lemma~\ref{lem:logloss}. (See
Appendix~\ref{app:logisticloss}.) Many related results are shown by
\citet{shamir2020logistic}, who also obtains tighter constants for
binary logistic regression.

For $UX$ larger than $\Omega(\sqrt{d}\ln T)$, \citet{foster2018logistic}
further show a lower bound that matches their upper bound up to the
dependence on $T$ inside the logarithm, for the case of binary logistic
regression with the $L_2$-norm\footnote{In restating their result, we
add a minimum with $T$, which appears to be missing from
\citet[Lemma~4]{foster2018logistic}.}:
\begin{theorem}[Lower Bound, \citet{foster2018logistic}]
   Consider binary logistic regression with the $L_2$-norm, $\Theta =
   \{\big(\begin{smallmatrix} \theta' \\ 0 \end{smallmatrix}\big) :
   \|\theta'\|_2 \leq U\}$ with $U = \Omega(\sqrt{d}\ln(T))$, and $\domX
   = \{x \in \reals^d : \|x\|_2 \leq 1\}$. Then
   \[
      \min_{\mathbf{Algs}} \max_\sample \max_{\theta \in \Theta} R_T(\theta)
        = \Omega\Big(d \ln \Big(\frac{U}{\sqrt{d}\ln T} \Big) \wedge T\Big).
   \]
\end{theorem}
We complement this by the following lower bound, which matches the upper
bound from Theorem~\ref{thm:logisticupper} for the regime where $UX$ is
smaller than $O(\sqrt{d})$:
\begin{theorem}[Lower Bound]\label{thm:lower_bound_logreg}
   Consider binary logistic regression with the $L_2$-norm, $\Theta =
   \{\big(\begin{smallmatrix} \theta' \\ 0 \end{smallmatrix}\big) :
   \|\theta'\|_2 \leq U\}$ and $\domX = \{x \in \reals^d : \|x\|_2 \leq
   X\}$ with $UX \leq 2 \sqrt{d}$ and $T > d$. Then the minimax regret
   is at least
   \begin{align}
     \min_{\mathbf{Algs}} \max_\sample \max_{\theta \in \Theta} R_T(\theta)
       &\geq d \ln \Big(\frac{UX\sqrt{T-d}}{4\sqrt{\pi}d} -
       \frac{2}{\sqrt{\pi (T/d-1)}}\Big)
       = \Omega\Big(d \ln \Big(\frac{UX\sqrt{T}}{d}\Big)\Big).
   \end{align}
\end{theorem}
\citet{shamir2020logistic} obtains related lower bounds, asymptotically
when $T \to \infty$, but he does not spell out their dependence on $U$
explicitly.
Combining Theorem~\ref{thm:logisticupper} with Lemma~\ref{lem:aggregate}
to adapt to $U$, and instantiating $\Theta_U = \{\theta \in \reals^{K
\times d} : \|\theta\| \leq U\}$ for concreteness, gives the following
scale-free adaptive result:
\begin{theorem}[Scale-Free, Adaptive]\label{adaptivelogisticregression}
  Let $\Theta = \reals^{K \times d}$ and $\domX = \reals^d$. Then, for
  any $\epsilon > 0$, there exists a scale-free strategy for the learner that
  guarantees 
  \[
    R_T(\theta)
      \leq 
          5 dK \ln \Big(\frac{2\|\theta\| X_T T}{dK} +
          \frac{\epsilon X_T}{X_{t^*}} + e\Big)
          +
          2\ln \Big(\log_2\Big(\frac{8\|\theta\|X_{t^*}T}{\epsilon dK} \bmax
          4\Big)\Big)
    \qquad
    \text{for all $\theta \in \reals^d$,}
  \]
  where $t^*$ is the first $t$ such that $\|x_t\| > 0$.
\end{theorem}
%
%We note in passing that
%Bayesian aggregation of Bayesian prediction strategies is equivalent to
%a single Bayesian prediction strategy with a hierarchical prior, 
%Bayesian prediction strategy with
%  hierarchical prior $\pi(\cdot |U)$ uniform on $\Theta_U$ and $\pi(U) =
%  \frac{1}{(\log_2(U/\Umin) + 1)(\log_2(U/\Umin)+2)}$ for $U \in \{\Umin
%  2^0, \Umin 2^1, \Umin 2^2,\ldots\}$
%
We see that the overhead for adaptivity becomes negligible if we can
take $\epsilon  = O(X_{t^*}/X_T)$. There is no automatic way available
to achieve this completely for free, because $X_T$ is unknown at the
start of the algorithm, but there are two reasonable solutions: the
first is to just take $\epsilon $ to be ``very small'', which is still
fine in the bound, because the $O(\log \log (1/\epsilon ))$ term hardly
grows with $1/\epsilon $. In particular, even for the smallest possible
positive value representable in a double precision floating point
number, e.g.\ $\epsilon  \approx 2.2\times 10^{-308}$, we still have
that $\ln(\log_2(1/\epsilon )) \leq 7$. The second solution is to
aggregate multiple copies of the algorithm using
Lemma~\ref{lem:aggregate} with $\epsilon  = e/\alpha = e X_{t^*}/X_T$,
$\alphamin = 1$, $\alphamax = \infty$. We then obtain the following
parameter-free result:
\begin{theorem}[Scale-Free, Adaptive,
Parameter-Free]\label{thm:adaptivelogistic_paramfree}
  Let $\Theta = \reals^{K \times d}$ and $\domX = \reals^d$. Then there
  exists a scale-free strategy for the learner that guarantees 
  \[
  R_T(\theta)
    \leq 
    5 dK \ln \Big(\frac{2\|\theta\| X_T T}{dK} +
    2e\Big)
    +
    2\ln
    \Big(\log_2\Big(\frac{16\|\theta\|X_T T}{edK} \bmax
    4\Big)\Big)
    +
    2\ln \Big(\log_2\Big(\frac{8X_T}{X_{t^*}}\Big)\Big)
  \]
  for all $\theta \in \reals^d$,
  where $t^*$ is the first $t$ such that $\|x_t\| > 0$.
\end{theorem}
The term $2\ln (\log_2(\frac{8X_T}{X_{t^*}}))$ is again unbounded in
theory, but it is at most $14$ when $X_T/X_{t^*}$ is restricted to the
range of double precision floating point numbers, which goes up to $1.8
\times 10^{308}$. This seems acceptable for all practical purposes.

\paragraph{Efficient Algorithms}

Since the Bayesian algorithm from Theorem~\ref{thm:logisticupper} is not
computationally efficient, efficient algorithms based on quadratic
approximations of the losses have been developed: by
\citet{Jezequel2020} for binary logistic regression and by
\citet{AgarwalKaleZimmert2021,Jezequel2021} for the multiclass case. (In
a different context, \citet{mourtada2022improper} also obtain
an efficient algorithm for misspecified \emph{offline} logistic regression.)
These efficient methods achieve worse regret rates, however, of order
$O(dKU \ln(T))$, with a linear rather than logarithmic dependence
on~$U$. The state of the art for the multiclass case is the algorithm of
\citet{AgarwalKaleZimmert2021}, which achieves the following run-time
and regret bound with respect to the $2 \rightarrow \infty$-norm of
$\theta$, which is defined as $\|\theta\|_{2,\infty} = \sup_{x : \|x\|_2
\leq 1} \|\theta x\|_\infty$:
\begin{theorem}[Non-adaptive, Efficient
Algorithm, \citet{AgarwalKaleZimmert2021}]\label{thm:logisticupper_efficient}
  Suppose the set of parameters is $\Theta = \{\theta \in \reals^{K \times d} :
  \|\theta\|_{2,\infty} \leq U\}$, and $\domX = \{x \in \reals^d :
  \|x\|_2 \leq X\}$. Then there exists a learning algorithm (depending
  on $U$ and $X$, but not on $T$) that achieves
  \[
    R_T(\theta) = O\Big(\big(UX + \ln K\big)dK\ln T\Big)
      \qquad \text{for all $\theta \in \Theta$,}
  \]
  and runs in time $O(d^2 K^3 + U X K^2 \ln(t(1+UX)))$ per round $t$.
\end{theorem}
The absolute constants in the theorem depend on a trade-off between
optimisation accuracy and run-time, which \citet{AgarwalKaleZimmert2021}
leave open; we assumed here that the optimisation accuracy is
$\textnormal{poly}(1/t)$.

Since the dependence on $U$ is now linear, adaptation to the norm of
$\theta$ becomes a much more pressing issue. We pursue this with
computational considerations in mind. Our starting point is the
observation that the algorithm from
Theorem~\ref{thm:logisticupper_efficient} is only computationally
efficient if both $UX \leq T^\beta$ and $d^2 K \leq T^\beta\ln(T)$ for
some small $\beta > 0$, in which case its run-time is $O(K^2
T^\beta\ln(T))$ per round. We will design a scale-free adaptive
algorithm with the same run-time by using the doubling trick to adapt to
$X \leq T^\beta/U$ and then applying Lemma~\ref{lem:aggregate} with $U =
\alpha$. In this case choosing $\alphamax$ to be finite is desirable for
computational reasons, because it reduces the number of copies of the
base algorithm that we need to run to
$\ceil{\log_2(\alphamax/\alphamin)}$. A good choice for $\alphamax$
exists, because $\alphamax = T^\beta/X_{t^*} \geq T^\beta/X_T$ is
sufficient to cover all $U$ such that $UX_T \leq T^\beta$, where $t^*$
is again the smallest $t$ such that $\|x_t\|_2 > 0$. We then choose
$\alphamin = T^{-\gamma}/X_{t^*}$ for $\gamma$ as large as possible
given the computational budget, leading to the following result:
\begin{theorem}[Scale-Free, Adaptive, Efficient
Algorithm]\label{thm:logistic_eff_ad_sf} Let $\Theta = \reals^{K \times d}$ and $\domX = \reals^d$.
Then, for any $\beta > 0$ and $c>0$, there exists a learner that
achieves
  \begin{multline}\label{eqn:logistic_efficient_adaptive}
    R_T(\theta)
        = O\Big(\Big(\|\theta\|_{2,\infty}X_T + \ln K +
        T^{- c T^\beta/(d^2 K)} \frac{X_T}{X_{t^*}} \Big)
          dK\ln T\Big) \\
          + 2\ln\Big(\log_2\Big(8\|\theta\|_{2,\infty}X_T T^{
\frac{c T^\beta}{d^2 K}}
          \bmax
          4\Big)\Big)
      \qquad
      \text{for all $\theta$ such that
      $\|\theta\|_{2,\infty} X_T \leq T^\beta$,}
  \end{multline}
  where $t^*$ is the first $t$ such that $\|x_t\|_2 > 0$. Furthermore,
  this learner runs in time complexity $O(d^2 K^3 + (1+c) K^2 T^\beta\ln(T))$ per round.
\end{theorem}
%
%The adaptive upper bound on the regret matches the non-adaptive rate
%from Theorem~\ref{thm:logisticupper_efficient}. With respect to
%run-time, we note that running the non-adaptive algorithm of
%\citet{AgarwalKaleZimmert2021} with $U = \Umax$ would already require
%$O(d^2 K^3T + \Umax X_T K^2 T\ln(T))$ computation steps. Assuming
%the first term in the run-time is dominant, we match this up to the
%logarithmic factor $\log_2(\Umax/\Umin)$. Since we cannot expect an
%efficient adaptive algorithm in regimes where even the non-adaptive
%method is no longer efficient, it is reasonable to take $\Umax \leq
%T^\beta$ for some $\beta > 0$ to ensure that the non-adaptive
%algorithm has polynomial run-time, in which case this logarithmic factor
%is of modest order $O(\log T)$.
%
Apart from the term involving $X_T/X_{t^*}$, this rate matches the
non-adaptive rate for any $\theta$ such that $\|\theta\|_{2,\infty}X_T
\leq T^\beta$, at no cost in the run-time compared to the efficient
non-adaptive algorithm (see discussion above). Like in
Theorem~\ref{thm:adaptivelogistic_paramfree}, the dependence on
$X_T/X_{t^*}$ is very minor: for $\beta = 1$, $c=10$ and
$\frac{X_T}{X_{t^*}} \leq 1.8 \times 10^{308}$ (the maximum value of a
double precision float), we then have $T^{- c T^\beta/(d^2 K)}
\frac{X_T}{X_{t^*}} \leq 1$ as soon as $T \geq 23 d^2 K$.
%
%\begin{align*}
%  T^{- 10 T/(d^2 K)} &\leq \frac{1}{1.8 \times 10^{308}}\\
%  T &\geq (d^2 K) \frac{\log_{10}(1.8 \times 10^{308})}{10 \log_{10} T}\\
%  T &\geq (d^2 K) \frac{\frac{1}{10} \log_{10}(1.8) + 30.8}{\log_{10}T}\\
%  T &\geq (d^2 K) \frac{31}{\log_{10}T}\\
%\end{align*}
%%
%Suppose $T \geq 23 d^2 K \geq 23$, then $\frac{31}{\log_{10}T} \leq 23$,
%so this is sufficient.

% Q. Is it possible to decrease $\alphamin$ further at no computation
% cost by making all $U \leq 1/(X_t t)$ predict the uniform distribution
% instead of their own distribution? By 2-Lipschitzness of the loss, the
% difference in cumulative loss can then be bounded by $t 2 \|U -
% 0\|_{2,\infty} X_t \leq 2$, so this comes for free in the regret bound,
% but we only ever have to run $\log_2(t)$ experts at any time.

\section{Square Loss and Least-Squares Regression}\label{sec:slls}
Consider the square loss
$
  \ell (a, y) =  \|a - y\|^2 / 2 \, 
$
over domains $\actions = \domY = \R^d$. 
For both prediction with the square loss (Section~\ref{sec:sq_loss_prediction}) and for regression (Section~\ref{sec:sq_loss_regression}), Exponential Weights (EW) with clipping adapts to the range of the data and to the norm of the comparator at essentially no cost.

The square loss is $1 / (4Y^2)$-mixable over $\actions = \domY = \ball(0, Y)$, with the mean as a substitution function. In contrast with the log-loss, the mixability constant depends on the domain of the data points, preventing the application of the EW aggregation scheme of Lemma~\ref{lem:mixable_losses} (in Appendix~\ref{app:logtomix}). We get around this issue thanks to a clipping trick of \citet{cutkosky2019artificial}.

For the square loss, there are two requirements to scale-freeness: the
predictions $a_t$ should \emph{not} change when all $x_t$ are scaled by
the same constant, but, if all $y_t$ are scaled by a constant, then the
predictions \emph{should} scale by the same constant.

\subsection{An Aggregation Procedure Tailored to the Square Loss}
In order to adapt to arbitrary hyperparameters without knowledge of the range of data $Y$, we propose a general aggregation scheme built upon the EW strategy with clipping. Throughout this section, we denote by $\Pi_Y$ the projection to the ball of radius $Y$.
\paragraph{Cutkosky Clipping}
To get around the issue of tuning the learning rate in the EW strategy,
we apply a trick from \citet{cutkosky2019artificial}, which is to feed
an algorithm clipped data points
\begin{equation}\label{eq:clipped_seq}
	\tilde{y}_{t} := 
          \frac{Y_{t-1}}{Y_t} y_t
          = \begin{cases}
            \frac{Y_{t-1}}{\NORM{y_t}} y_t & \text{if $\NORM{y_t} \geq
            Y_{t-1}$}\\
            y_t & \text{otherwise,}
          \end{cases}
\end{equation}
where $Y_t = \max_{s\leq t} \| y_s\|$. Then the algorithm knows in advance that its next data point will be bounded by $Y_{t-1}$. This is exactly the knowledge needed to tune the learning rate when using EW.
\emph{A priori}, one would need to ensure that the actions of the experts are bounded by $Y_{t-1}$ to satisfy mixability. It turns out this is not necessary: it suffices to also feed clipped actions to EW.

\paragraph{Adapting}
Suppose that $A(\alpha)$ is an algorithm for the square loss that depends on a hyperparameter $\alpha \in [\alphamin, \alphamax)$, with $0 \leq \alphamin \leq \alphamax \leq \infty$. Consider the grid of parameter values $\{2^m \alphamin |\, m = 0, \dots, M \}$ where $M = \ceil{\log_2(\alphamax/\alphamin)}$. For a sequence of data points $(y_t, x_t)$, denote by $a_{t, \alpha}$ the output of $A(\alpha)$ at time $t$. Apply the EW strategy based on the clipped data points $\tilde y_t$ and the clipped actions $\Pi_{ Y_{t-1}}(a_{t, \alpha})$, with learning rate $\eta_t = 1 / (4 Y_{t-1}^2)$, and prior $\pi(m) = \frac{M+2}{(M+1)(m + 1)(m+2)}$. The next result is an analogue to Lemma~\ref{lem:aggregate} for the square loss.
\begin{lemma}
  \label{lem:agg_sq_loss}
  Let $A(\alpha)$ be an algorithm that achieves a regret bound
  $B_T(\theta,\alpha) \geq R_T(\theta)$ for any $\alpha \in [\alphamin, \alphamax)$ and $\theta \in
  \Theta$. Then, it is possible to adapt to $\alpha \leq \alphamax$ with regret
  bounded by
  \begin{equation}
		R_T(\theta)
		\leq \max_{\alpha' \in [\alpha,2\alpha \bmax \alphamin]} B_T(\theta,\alpha') + 8Y_T^2 \ln \log_2\Big(\frac{8\alpha}{\alphamin}  \bmax 4\Big)
		+ 2Y_T^2 \, \quad \text{for all $\theta \in \Theta$} \, .
	\end{equation} 
  Furthermore, if all algorithms $A(\alpha)$ are scale-free, then the
  aggregated procedure is scale-free as well.
\end{lemma}
\subsection{Square Loss Prediction}\label{sec:sq_loss_prediction}
Let us apply the results built above to the case of prediction with the
square loss. In this case, $\Theta = \actions = \R^d$, there are no
observed features (i.e.\ $\domX$ is a singleton), and the actions functions are $a_\theta \equiv \theta$. Slightly abusing notation, we shall denote $\ell(\theta, y) = \| \theta - y \|^2 / 2$.
\paragraph{Aggregated Gradient Descent}
We apply the aggregation procedure of
Lemma~\ref{lem:agg_sq_loss} to multiple instances of Online Gradient
Descent (cf.\ Appendix~\ref{app:gradient-descent}) with step sizes $1 /
(\lambda+  t)$, where we aggregate over $\lceil \log_2 T \rceil$ values
of $\lambda = \alpha$ from $\alphamin=1$ to $\alphamax=T$.
Note that the clipping of the actions has no effect in this case, since the individual updates of every expert are already in $\ball(0, Y_{t-1})$ at every $t$. 
%Each instance chooses an action $a_t(\alpha)$. The aggegation procedure builds a distribution $P_{t}$. The action $\zeta(P_t)$ is played. The data point $y_t$ is given. Then $\tilde y_t$ is fed to each instance, and each instance suffers the loss $\ell(a_t(\alpha), \tilde y_t)$. Then each algorithm chooses the next action.
\begin{theorem}\label{thm:sq_loss_full_bound}
	In square loss prediction, there exists a scale-free algorithm such that for any $\theta \in \R^d$, 
	\begin{equation*}
		R_T(\theta) 
		\leq 
		2 Y^2_T \log \bigg(2 +  T \Big(\frac{\|\theta\|^2}{Y_T^2}\bmin 1 \Big)\bigg) 
    + 8Y^2_T\ln \log_2\bigg(\frac{8 Y_T^2}{ \|\theta\|^2 }\bmax 8 \bigg) + 3Y_T^2 \, .
	\end{equation*}
\end{theorem}
The per-round computation time is $\lceil \log_2 T \rceil$ times the cost of a gradient descent update, with guarantees that match the non-adaptive lower bound of Theorem~\ref{thm:sq_lower_bound} for $d=1$ and constant features.

\subsection{Least-Squares Regression}\label{sec:sq_loss_regression}
Upon observing a feature point $x_t \in \domX$, the learner outputs a predictions $a_t \in \actions =  \R$, then receives the answer $y_t \in \domY = \R$.  The learnes competes against functions $\theta \in \mathcal F$ where  $\mathcal F\subset \domY^{\domX}$ is a known set of comparator functions. 
We assume  $\mathcal F$ is a separately implementable Reproducing Kernel Hilbert Space (RKHS) with the kernel $k$ over $\domX$ \citep[Definition~1]{gammerman2004on-line}. % and \cite{scholkopf2002learning} for a more detailed account of kernel learning).
We refer to the kernel matrix $K_T = (k(x_s, x_t))_{(s, t)\in [T]^2}$ and to $X_t = \max_{s \in [t]}k(x_s, x_s) $.%, which generalise the Gram matrix and the runnning maximum of the squared norm of the features from regression in $\R^d$, respectively. 

Kernel methods are useful when the algorithm depends on the feature
vectors $x_1,x_2,\ldots$ exclusively via the quantities $k(x_s, x_t)$, thanks to the kernel trick. For such learners, $X$-scale-invariance generalises to invariance by scaling of $k$ by a positive factor.

Analyses of (non-)parametric regression focus on the asymptotic
dependence on the number of data points $T$. When $\mathcal F$ is a rich
class of functions, even for the parametric case in large dimension, the
norm of the comparator impacts the rates, hence the importance of
adaptation. We show that adapting to $\|\theta\|$ comes at no cost on
the regret, with a scale-free algorithm.

\paragraph{Aggregated-KAAR} 
A key algorithm in online regression is the Azoury-Vovk-Warmuth forecaster (also called the forward algorithm) from \citet{AzouryWarmuthRelativelossODE, Vovk98competitiveregression}, and its kernelised version KAAR \citep{gammerman2004on-line}. This algorithm has been analysed and modified in a variety of settings; see, e.g., \citet{orabona2015generalized, jezequel2019efficient, gaillard2019uniform, jezequel2019efficient, zadorozhnyi2021online}.
Notably, for linear regression \citet{gaillard2019uniform} notice that the 
Vovk-Azoury-Warmuth forecaster with regularization parameter set to
$\lambda = 0$ is scale-free and enjoys a regret bound that is optimal up
to an additive term that is a constant for many reasonable sequences of
feature vectors $x_t$, but can potentially blow up. Upon seeing $x_t$,
KAAR with regularisation $\lambda > 0$ predicts $a_t = \theta_t(x_t)$
where $\theta_t \in \mathcal F$ is picked according to the rule
\begin{equation}
	\theta_t = \underset{\theta \in \mathcal F}{\mathrm{argmin}} \bigg\{ \frac{\lambda \|\theta\|^2}{2} + \frac{1}{2}\theta(x_t)^2 +\frac{1}{2}\sum_{s=1}^{t-1}\left(\theta(x_s)-y_{s}\right)^{2}  \bigg\} \, .
\end{equation}
The update $\theta_t$ admits the closed-form expression 
\begin{equation}\label{eq:kaar-updates}
	\theta_t(x) =  (y_1,\dots, y_{t-1}, 0 )^\top ( \lambda I_t +  K_t)^{-1}\big(k(x_1, x), \dots , \, k(x_t, x)\big) \, .
\end{equation}
%\algname{KAAR} requires the tuning of a regularisation parameter, typically set to optimise the asymptotic dependence on $T$ given an upper bound $U$ on the norm of the target comparator.
To ensure scale-invariance, we tune the regularisation proportionally to the first non-zero feature (and predict $0$ until there is one); we call the ensuing algorithm \algname{KAAR-sf}$(\alpha)$, which will be the building block in our aggregated algorithm.
%
%While being a desirable property, scale-invariance on its own is not enough to make an algorithm useful. In our case, \algname{KAAR-sf}$(\alpha)$ is sensitive to the tuning of $\alpha$ and to the initial feature $X_{t^\star}$.

%Let us perform the aggregation before we discuss specific examples in which \eqref{eq:kaar-regret-bound} is easier to interpret. 
For clarity, let us ignore computational issues and run an infinite number of instances. We apply Lemma~\ref{lem:agg_sq_loss} twice in order to aggregate both arbitrary small and arbitrary large values of~$\alpha$. Doing so, the dependence on the initial guess of the correct scale moves into a $\log \log$ factor. 
\begin{theorem}[Adaptive, Scale-Free]\label{thm:akaar-general}
	In kernel least-squares regression, there exists a scale-free algorithm such that for any $\lambda > 0$,
	\begin{equation}\label{eq:a-kaar-general}
			R_T(\theta)
			\leq \frac{\lambda \|\theta\|^2}{2} 
			+ \frac{Y_T^2}{2} \ln \det \bigg( I_T + \frac{1}{\lambda }  K_T \bigg)
			+ 8Y_T^2 \ln \left( \frac{3e}{2} \left| \log_2  \bigg(\frac{\lambda}{X_{t^\star}^2}  \bigg)\right| \right)\,.
	\end{equation}
	where $Y= \max_{t \in [T]} |y_t|$ and $K_t = \big(k(x_u, x_v)\big)_{(u, v)\in [T]^2}$.
\end{theorem}
This new algorithm, which we call \algname{A-KAAR}, is scale-free and enjoys the same guarantees as best tuned \algname{KAAR}, up to a small $\log\log$ term. We state a first consequence that holds for any kernel.
\begin{corollary}[Dimension-Free, Scale-Free]\label{cor:dim-independent-KAAR}
	\algname{A-KAAR} enjoys the dimension-free regret bound
	\begin{equation}\label{eq:akaar-dim-independent}
		R_T(\theta) 
		\leq 
		 X_TY_T\|\theta\|\sqrt T 
			+ 8Y^2_T \ln \bigg( \frac{3e}{2} \bigg|\log_2 \bigg( \frac{ X_TY_T\sqrt T}{\|\theta\|X_{t^\star}^2} \bigg) \bigg| \bigg) \, . 
	\end{equation}
\end{corollary}
In the worst-case, the regret of \algname{A-KAAR} grows at most at an $XY\|\theta\|\sqrt{T}$ rate. This matches the non-adaptive lower bound of Theorem~\ref{thm:sq_lower_bound} in the large-dimensional regime. Faster rates are achievable under additional assumptions on $\mathcal F$, as we shall see now.
\subsubsection{Parametric Regression}\label{sec:parametric-least-squares}
In this case, the set $\mathcal F$ is the set of linear functions over $\R^d$, identified with $\R^d$. Then $\domX = \R^d$ and $k(x, x') = \langle x, x' \rangle$ and $X_t = \max_{s \in [T]} \|x_s\|$ and the RKHS norm is the Euclidean norm. \algname{KAAR} specialises to the VAW forecaster, and its aggregated version enjoys the adaptive upper bound:
\begin{theorem}\label{thm:regression-finite-dim}
	In $d$-dimensional linear regression, \algname{A-KAAR} guarantees that for any $\theta \in \R^d$, 
	\begin{equation*}
		R_T(\theta) \leq \frac{dY_T^2}{2} \ln \biggl( 1 + \frac{T\|\theta\|^2X_T^2}{d^2Y_T^2}  \biggr) + \frac{dY_T^2}{2}
		 + 8Y_T^2 \ln \left( \frac{3e}{2} \bigg|\log_2 \left( \frac{ dY_T^2}{\|\theta\|X_{t^\star}^2}\right)  \bigg| \right)\, .
	\end{equation*}
\end{theorem}
The bound matches the non-adaptive lower bound of
Theorems~\ref{thm:sq_lower_bound}. Note that this implies a uniform
regret bound over comparators in $\R^d$, by instantiating the comparator
$\theta$ to be a minimiser $\theta^\star$ of the least-squares error on
the data, i.e.\ a maximiser of the regret. See Corollary~3 in \citet{gaillard2019uniform} and its proof for upper bounds on the norm of $\theta^\star$. 
% Hedi: I'm reluctant to write more details than this, because it requires notation 
% and case distinctions but is not that interesting imo. (Precisely, we would need to talk about the smallest non-negative eigenvalue of the Gram matrix, and consider apart cases where the norm of \theta_\star is extremely small.)
%
\subsubsection{Comparator-Adaptive bounds under the Capacity Condition}
In typical uses, $\mathcal F$ is vastly richer than a set of linear
functions. The effective dimension $d_{\mathrm{eff}}(\lambda)$ of the
features $K_T$ at scale $\lambda$ \citep{zhang2003effective} provides a
standard data-dependent complexity measure of $\mathcal F$ (cf.
Appendix~\ref{app:rkhs}). The space $\mathcal F$ is said to satisfy the $\gamma$-capacity condition if for any sequence of features of length $T$ and for any $\lambda > 0$, the effective dimension grows at most at a rate of $d_{\mathrm{eff}}(\lambda) \leq (C_k T / \lambda)^\gamma$ for some $C_k > 0$. Under this condition, the second term in \eqref{eq:a-kaar-general} is polynomial in $T/\lambda$, and we obtain the following rates.
\begin{theorem}\label{thm:agg-KAAR-RKHS}
	If $\mathcal F$ satisfies the $\gamma$-capacity condition, then \algname{A-KAAR} guarantees thhat
	\begin{equation*}
		R_T(\theta) \leq \tilde{\mathcal O} \big( Y_T^{2 / (1+ \gamma)} \|\theta\|^{2\gamma / (1+\gamma)} T^{\gamma / (1+ \gamma)} \big)  \quad \text{ for any } \theta \in \mathcal F \, .
	\end{equation*}
  A finite-time version of the bound depending on $X_T$, $X_{t^\star}$ is available in~\eqref{eq:full-kernel-adapt-bound}, in Appendix~\ref{app:rkhs} .
\end{theorem}
% The units are right: if $C_k$ has the units of [X] and [Y] = [X][\theta], then R_T has unit [Y]^2 .
The capacity condition is satisfied, e.g.\ when $\mathcal F$ is a space of smoothing splines \citep[Section 4]{zhang2003effective}, or for Sobolev spaces; we detail this application in the next section.
\paragraph{Comparator-Adaptive Regression over Sobolev Spaces}
By Theorem~3 of \citet{zadorozhnyi2021online}, the results above imply adaptive rates when the class of functions is the Sobolev space $W_{s,p}([-1, 1]^d)$ with $p\geq 2$; we refer the reader to \citet{adams2003sobolev} for definitions and properties of Sobolev spaces, and to \citet[Chapter~10]{wendland2004scattered} for more details on Sobolev spaces as RKHS. 
%By the Sobolev embedding theorem, when $s \geq d / p$, any element of $W_{s, p}([-1, 1]^d)$ admits a unique continuous representative. Regret is always meant with respect to this representative. 
For simplicity, let us state the results in the case when $s$ is an integer and $s \geq d / 2$. For fractional $s$, the same rates are valid, up to a $T^\eps$ factor with $\eps$ arbitrarily small, and the rates change when $s< d/2$. 
\begin{corollary}\label{thm:kaar-sobolev}
	For $\mathcal F = W_{s, p}([-1, 1]^2)$ for $s \in \NN$ with $s \geq d/2$, there exists an algorithm such that 
  \begin{equation*}
    R_T(\theta) \leq \tilde{\mathcal O}  \Big( Y_T^{4s / (2s + d)} \|\theta\|_{s, p}^{2d / (2s + d)} T^{d / (2s + d)}\Big)  \quad \text{for any } \theta \in W_{s, p}([-1, 1]^2)\, .
  \end{equation*}
\end{corollary}
The exponent on $T$ is optimal \citep[Thm~9]{rakhlin2014online, zadorozhnyi2021online}. 
% Moreover \citet{zadorozhnyi2021online} also prove that for $s \in  [d/p, d / 2]$, \algname{KAAR} run over the base space $W_{d/2+\eps, 2}([-1, 1]^d)$ with $\lambda$ tuned as a function of $s$ obtains the minimax rates of order $T^{1-s/d}$. \algname{A-KAAR} would automatically adapt to the smoothness $s \in [d / p, d /2]$ in that case.
\paragraph{Efficient Methods}
The updates \eqref{eq:kaar-updates} can be computed in $\mathcal O (t^2)$ time and memory, and both these complexities can be improved for specific kernels. % In the d-dimensional case, the time complexity can be improved to $\mathcal O (d^2)$. 
Running EW over experts \algname{KAAR-sf}$(\alpha)$ with small $\alphamin$ and large $\alphamax$ would be an implementable strategy in cases when single instances are efficient, at the cost of limiting the adaptivity to a specific range of values of $\|\theta\|$.
\cite{jezequel2019efficient} build a faster version of \algname{KAAR} enjoying essentially the same regret bound as \eqref{eq:kaar-regret-bound}, but with better computational complexity for large $T$. The aggregation we propose applies to this algorithm too, and would incorporate the improvements in computational complexity. %Furthermore, the same analysis can be carried with Gaussian kernels, for which the effective dimension grows as $(\log n)^{\gamma}$ instead of $n^\gamma$, giving rise to polylog rates instead of polynomial rates.

%
% \paragraph{OCO with strongly convex lossess}
% \begin{theorem}\label{stronglyconvexadaptive}
% 	In OCO with $\mu$-strongly convex losses, the player can ensure that 
% 	\begin{equation*}
% 		R_T(\theta)
% 			\leq \frac{9G^{2}}{8\mu}+\frac{G^{2}}{2\mu}\mathrm{log}\left(1+\frac{8\mu\NORM{\theta}^{2}T}{G^{2}}\right)+\frac{G^{2}}{\mu}2\ln\left(\mathrm{log}_{2}\left(\frac{G^{2}}{\mu\NORM{\theta}^{2}}\right)\right). 
% 	\end{equation*}
% 	where $G = \max \|g_t\|$.
% \end{theorem}
%
\subsection{Lower Bounds}
The main lower bound for the square loss, in Theorem~\ref{thm:sq_lower_bound}, provides the optimal asymptotic rate together with the dependence on the comparator. While the dependence on $T$ is a standard result in the literature \citep{takimoto2000the-minimax, abernethy2008optimal, hazan07, gaillard2019uniform}, we did not find a version of the lower bound that provided the dependence on $U$. We thus refined the proof of \citet[Theorem~2]{Vovk01CompetitiveOS}. 
In the large-dimensional regime where $ d \geq T (UX/Y)^2$, the finite-time version of the first lower bound is vacuous. The effect of the curvature becomes negligible and rates behave like in the linear loss case, as shown in the upper bound (Corollary~\ref{cor:dim-independent-KAAR}) and in the matching lower bound.
\begin{theorem}\label{thm:sq_lower_bound}
	Fix $X, Y, U > 0$. In linear least-squares regression over $\R^d$, if $d( Y / (UX))^2 \geq 1$, then for any algorithm, 
  \vspace{-2ex}
	\begin{equation*}
		\sup_{\mathcal S}\sup_{\theta\in\ball(0,U)}R_{T}(\theta)\geq \left\{
    \begin{aligned}
     & \frac{dY^2}{2}\log \biggl(  \frac{TU^2X^2}{d^2Y^2}  \biggl) 
       + \mathcal O (\log \log T) && \text{as} \quad T \to \infty \, , \\
     & (\sqrt 2 / 8) UXY \sqrt{ T} \quad &&\text{for} \quad T \leq (d/8) (Y / (UX))^2   \, .
    \end{aligned}\right.
	\end{equation*}
	Moreover, if $d = 1$, the first bound holds with $x_t = X$ for all $t$.
\end{theorem}
A finite-time version of the first bound can be found in
\eqref{eq:vovk-lower-bound-T-dep} in Appendix~\ref{app:lower-bounds-sq}.
\section{Discussion, Conclusions and Future Work}

We have shown that scale-free algorithms can adapt to the norm of the
comparator at almost no cost in common learning scenarios. While we have
endeavored to complete the story, some points remain open. We note that
the case of strongly convex losses with fixed strong-convexity parameter
$\mu$ should be treatable by a proof
directly analogous to that of Theorem~\ref{thm:sq_loss_full_bound}.
Time-varying $\mu$ or strong convexity with respect to other Bregman
divergences (c.f.\ \citep{rakhlinetal}) would not be as easy,
because the former would affect the mixability of the loss, and the
latter might break~\eqref{eqn:doubleproj}.
Additionally,
%while we have efficient algorithms for logistic
%regression and square loss, these are yet to be found for Besov spaces. 
avoiding $O(\log \log
\frac{X_T}{X_{t^*}})$ terms in the logistic/least squares linear
regression cases would be desirable, at least from a theoretical
perspective; \citet{gerchinovitz11a} also observe this, while
\citet{gaillard2019uniform} avoid it but in exchange find a different
complicated dependence on the features. Lastly, it would be of interest
to have comparator-adaptive lower bounds for regression in Sobolev
spaces, with an explicit dependence on both $Y$ and $\NORM{\theta}_{s,p}$.

% aowledgments---Will not appear in anonymized version
\acks{All authors were supported by the Netherlands Organization for Scientific Research
(NWO) under grant number VI.Vidi.192.095.}

\bibliography{rangeadapt.bib}
\appendix

\newcommand{\interproofspace}{\vspace{5mm}}

\section{Proofs from Sections~\ref{sec:loglossandregression}}

\subsection{Proofs for the Aggregation}

Guarantees for our aggregation scheme derive from a straightforward application of the following standard guarantee for the Bayesian prediction strategy (see, e.g., Section~10 in \cite{dawid1984present} or Lemma 2.1 of \cite{kakade2004online}).
\begin{lemma}\label{lem:logloss}
  The Bayesian prediction strategy with prior $\pi$ achieves
  \begin{equation}
    \sum_{t=1}^{T} \logloss(p_t,y_t) \leq \E_{\theta \sim
    \gamma}\bigg[\sum_{t=1}^T \logloss(p_{\theta,t},y_t)\bigg]+\KL(\gamma\| \pi)
    \qquad
    \text{for all distributions $\gamma$,}
  \end{equation}
  with equality if $\gamma = \pi(\theta \mid \sample_T)$.
\end{lemma}
\begin{proof}%{\bfseries of Lemma~\ref{lem:logloss}}
By telescoping, the cumulative loss of the Bayesian prediction strategy
simplifies to
\begin{align*}
  \sum_{t=1}^T \logloss(p_t,y_t)
    &= \sum_{t=1}^T -\ln
        \frac{\int \prod_{s=1}^t p_{\theta,s}(y_s) \intder
              \pi(\theta)}
             {\int \prod_{s=1}^{t-1} p_{\theta,s}(y_s) \intder
              \pi(\theta)}
    = -\ln \int \prod_{t=1}^T p_{\theta,t}(y_t) \intder
    \pi(\theta)\\
    &= -\ln \int e^{- \sum_{t=1}^T \logloss(p_{\theta,t},y_t)} \intder \pi(\theta).
\end{align*}
The result then follows by recognising the right-hand side as minus the convex
conjugate of the Kullback-Leibler divergence (i.e., by applying the Donsker-Varadhan
lemma \cite[Corollary 4.14]{boucheron2013concentration}).
\end{proof}

\interproofspace

\begin{proof}{\bfseries of Lemma~\ref{lem:aggregate}}
With minor abuse of notation, let $p_{\alpha,t}$ denote the prediction
of $A(\alpha)$. Then, given any $\alpha \in [0,\alphamax)$, let $\alpha^*$ be the
smallest value in the grid exceeding $\alpha$, such that $\alpha \leq
\alpha^* \leq 2\alpha \bmax \alphamin$, and let $m^* = \log_2
(\alpha^*/\alphamin)$. By Lemma~\ref{lem:logloss}, with $\gamma$ a
point-mass on $\alpha^*$, we find that
\begin{equation*}
  \sum_{t=1}^{T} \logloss(p_t,y_t)
    \leq \sum_{t=1}^T \logloss(p_{\alpha^*,t},y_t)-\ln \pi(m^*) \, .
\end{equation*}
Then, for any $\theta \in \Theta_\alpha \subseteq \Theta_{\alpha^*}$,
\begin{align*}
  R_T(\theta)
    &\leq B_T(\theta,\alpha^*) -\ln \pi(m^*)
    < B_T(\theta,\alpha^*) + 2\ln (m^*+2)\\
    &= B_T(\theta,\alpha^*) + 2\ln
    \Big(\log_2\big(\frac{\alpha^*}{\alphamin}\big) +2\Big)
    \leq \max_{\alpha' \in [\alpha, 2\alpha \bmax \alphamin]}
    B_T(\theta,\alpha') + 2\ln\Big(\log_2\big(\frac{8\alpha}{\alphamin}
    \bmax 4\big)\Big),
\end{align*}
as required.
\end{proof}

\subsection{Proofs for the Normal Location Family}

\begin{proof}{\bfseries of Theorem~\ref{thm:normallocation}}
Abbreviate $y^t = (y_1,\ldots,y_t)$, define $p_\theta(y^T) =
\prod_{t=1}^T p_{\theta,t}(y_t)$, let $\mlmu(y^T) = \argmax_{\theta
\in \reals^d} p_\theta(y^T) = \frac{1}{T} \sum_{t=1}^T y_t$ be the
unconstrained maximum likelihood, and take $\mltheta(y^T) =
\argmax_{\theta \in \ball(0,U)} p_\theta(y^T)$ to be the maximum
likelihood restricted to $\Theta$, which is the projection onto
$\ball(0,U)$ of $\mlmu$: $\mltheta(y^T) = \Proj_{\ball(0,U)}(\mlmu) =
\min\{1,\frac{U}{\|\mlmu\|}\} \mlmu$.

Since the horizon $T$ and the predictions $p_{\theta,t}$ for
$t=1,\ldots,T$ are known in advance, the exact minimax strategy
\citep{grunwald2007minimum} is to predict $p_t = \pnml(y_t \mid
y^{t-1})$, where
\[
  \pnml(y^T) = \frac{p_{\mltheta(y^T)}(y^T)}{Z}
\]
is the normalised maximum-likelihood (NML) density, with normalising constant
\[
 Z = \int_{\reals^{d \times T}} p_{\mltheta(y^T)}(y^T) \intder y^T.
\]
The NML density is an equalizing strategy that ensures the regret is
exactly
\[
  \max_{\theta \in \ball(0,U)} R_T(\theta) = \ln Z
\]
for all sequences $y^T$. The value $\ln Z$ is called the
\emph{stochastic complexity}. It therefore remains to evaluate the
integral $Z$. To this end, we use that $\mlmu(y^T)$ is a sufficient
statistic for $y^T$, which means that the conditional density of
$p_\theta$ given $\mlmu$ does not depend on $\theta$. We therefore
define $\bar{p}(y^T \mid \mlmu(y^T)) := p_\theta(y^T \mid \mlmu(y^T))$,
independently of $\theta$. Consequently, $p_{\mltheta(y^T)}(y^T) =
\bar{p}(y^T \mid \mlmu)p_{\mltheta(\mlmu)}(\mlmu)$, where
$\mltheta(\mlmu) := \argmax_{\theta \in \Theta} p_\theta(\mlmu)$ equals
$\mltheta(y^T)$ for any $y^T$ for which $\mlmu = \mlmu(y^T)$. Hence
\begin{align*}
  Z
    &= \int_{\reals^d} \E_\nu\Big[
    p_{\mltheta(y^T)}(y^T)
    \Bigm| \mlmu(y^T) = \mlmu
    \Big] \intder \mlmu
    = \int_{\reals^d} p_{\mltheta(\mlmu)}(\mlmu) \E_\nu\Big[
    \bar{p}(y^T \mid \mlmu)
    \Bigm| \mlmu(y^T) = \mlmu
    \Big] \intder \mlmu\\
    &= \int_{\reals^d} p_{\mltheta(\mlmu)}(\mlmu) \intder \mlmu
    = \frac{1}{(2\pi \sigma^2/T)^{d/2}}
      \int_{\reals^d} e^{-\frac{\|\mlmu -
      \Proj_{\ball(0,U)}(\mlmu)\|_2^2}{2
      \sigma^2/T}} \intder \mlmu,
\end{align*}
where the first identity is the law of total probability for Lebesgue
measure, the third identity uses that $\bar{p}(y^T \mid \mlmu)$
integrates to $1$ over its domain, and the last identity comes from the
fact that, for any~$\theta$, $\mlmu$ is the average of $T$ normal
distributions $\normaldist(\theta,\sigma^2 I)$, and is therefore
distributed as $\normaldist(\theta,\tfrac{\sigma^2}{T} I)$.

We evaluate the remaining integral, starting from the observation that
$\|\mlmu - \Proj_{\ball(0,U)}(\mlmu)\|_2 = \max\{\|\mlmu\| - U,0\}$
depends only on the length of $\mlmu$. For $d=1$, computing $Z$ is
straightforward, so assume for the remainder that $d \geq 2$. Switching
to hyperspherical coordinates with radial parameter $r \in [0,\infty)$
and angular parameters $\phi \in \Phi := [0,\pi]^{d-2} \times [0,2\pi]$,
then implies that
\begin{align*}
  \int_{\reals^d} e^{-\frac{\|\mlmu - \Proj_{\ball(0,U)}(\mlmu)\|_2^2}{2 \sigma^2/T}} \intder \mlmu
    &= \int_0^\infty \int_\Phi e^{-\frac{\max\{r-U,0\}^2}{2 \sigma^2/T}}
    r^{d-1}  \prod_{i=1}^{d-2} \sin^{d-1-i}(\phi_i) \intder \phi \intder r\\
    &= \int_0^\infty  e^{-\frac{\max\{r-U,0\}^2}{2 \sigma^2/T}}
    r^{d-1} \intder r \times \int_\Phi \prod_{i=1}^{d-2}
    \sin^{d-1-i}(\phi_i) \intder \phi.
\end{align*}
The second factor evaluates to
\[
\int_\Phi \prod_{i=1}^{d-2} \sin^{d-1-i}(\phi_i) \intder \phi
    = d \int_0^1 \int_\Phi r^{d-1}
    \prod_{i=1}^{d-2} \sin^{d-1-i}(\phi_i) \intder \phi \intder r
    = d \Vol(\ball(0,1))
    = \frac{d \pi^{d/2}}{\Gamma(\tfrac{d}{2}+1)},
\]
where $\Vol(\ball(0,r)) = \frac{\pi^{d/2} r^d}{\Gamma(\tfrac{d}{2}+1)}$
is the volume of a ball of radius $r$; the first factor can be
re-expressed as
\begin{align*}
  \int_0^\infty  e^{-\frac{\max\{r-U,0\}^2}{2 \sigma^2/T}}
    r^{d-1} \intder r
      &= \int_0^U r^{d-1} \intder r
      + \int_U^{\infty} r^{d-1} e^{- T(r- U)^2/(2\sigma^2)} \intder r
      \\
      &= \frac{U^d}{d} + \int_0^{\infty} (r+U)^{d-1} e^{-Tr^2/(2 \sigma^2)} \intder r \\
      &= \frac{U^d}{d} + \frac{\sigma}{\sqrt T} \int_0^\infty \Big(
      \frac{r \sigma }{\sqrt T}+U\Big)^{d-1} e^{-r^2/2} \intder r.
\end{align*}
Putting all equalities together establishes \eqref{eqn:logminmax}. The
proof is completed upon observing that $V(U,T) = O(1/\sqrt{T})$.
\end{proof}

\subsection{Proofs for the Logistic Loss}\label{app:logisticloss}

\begin{proof}{\bfseries of Theorem~\ref{thm:logisticupper}}
Consider the statement of Lemma \ref{lem:logloss}
\begin{equation*}
    \sum_{t=1}^{T} \logloss(p_t,y_t) \leq \E_{\theta \sim
    \gamma}\Big[\sum_{t=1}^T \logloss(p_{\theta,t},y_t)\Big]+\KL(\gamma\| \pi)
    \qquad
    \text{for all distributions $\gamma$}.
  \end{equation*}
Set $\gamma = \pi(\cdot|A)$, where $A:=\left\{ a\theta^{*} + (1-a)\theta |\theta \in\Theta \right\} \subset \Theta$ where $a\in[0,1)$ as in \cite{foster2018logistic} and $\pi$ is chosen to be uniform. Then the Kullback-Leibler divergence reads
\begin{equation*}
\KL(\pi(\cdot|A) \| \pi) = \int_{\theta\in\Theta}\ln\left( \frac{\mathrm{d}\pi(\theta|A)}{\mathrm{d}\pi(\theta)}\right)\mathrm{d}\pi(\theta|A)=-\ln\pi(A) = \ln\frac{\mathcal{V}(\Theta)}{\mathcal{V}(A)}=d_{\Theta}\ln\frac{1}{1-a},
\end{equation*}
where the last equality follows from $\mathcal{V}(A)=(1-a)^{d_{\Theta}}\mathcal{V}(\Theta)$, and so we can bound
\begin{align*} 
\E_{\theta \sim
    \pi(\cdot|A)}\Big[\sum_{t=1}^T \logloss(p_{\theta,t},y_t)\Big]&\leq\max_{\theta\in\Theta}\Big\{\sum_{t=1}^T \logloss(p_{a\theta^{*}+(1-a)\theta,t},y_t)\Big\} \\
    &\leq\sum_{t=1}^T \left(\logloss(p_{\theta^{*},t},y_t)+4(1-a)B\NORM{x_{t}}_{*}\right), 
\end{align*}
where in the second inequality we have used the $2$-Lipschitzness of the logistic loss with respect to the $L_{\infty}$-norm (\cite{foster2018logistic}, Lemma 1) and their observation that $\NORM{\left(a\theta^{*}+(1-a)\theta - \theta^{*}\right)x_{t}}_{\infty}=(1-a)\max_{k\in[K]}|\IP{\theta_{k}-\theta^{*}_{k}}{x_{t}}|\leq 2(1-a)U\NORM{x_{t}}_{*}$ for all $\theta\in\Theta$. It follows that
\begin{equation*}
\sum_{t=1}^{T} \logloss(p_t,y_t) \leq \sum_{t=1}^T \left(\logloss(p_{\theta^{*},t},y_t)+4(1-a)U\NORM{x_{t}}_{*}\right)+d_{\Theta}\ln\frac{1}{1-a}.
\end{equation*}
Setting $1-a:=1\wedge\frac{d_{\Theta}}{U\sum_{t=1}^{T}\NORM{x_{t}}_{*}}$ and bounding appropriately completes the proof.
\end{proof}

\interproofspace

\begin{proof}{\bfseries of Theorem~\ref{thm:lower_bound_logreg} \; }
  We first lower bound the minimax regret by restricting the maximum
  over $\theta$ to $\Theta' = \{\big(\begin{smallmatrix} \theta' \\ 0
  \end{smallmatrix}\big) : \|\theta'\|_\infty \leq U/\sqrt{d}\} \subset
  \Theta$. To construct a hard data sequence $\sample$, we then set $x_t
  = X e_{(t \;\textnormal{mod}\; d) + 1}$, which reduces the learning
  task to $d$ independent one-dimensional learning tasks with $n \geq
  \tfrac{T}{d}-1$ learning rounds each. Consequently, our lower bound
  will be $d$ times a lower bound that holds for each of the
  one-dimensional tasks.
  
  So consider a one-dimensional binary logistic regression task with
  $|\theta| \leq U/\sqrt{d}$ and $x_t = X$ for all $t = 1,\ldots,n$.
  This is equivalent to log loss prediction of $z_t = (y_t + 1)/2 \in
  \{0,1\}$ with respect to the Bernoulli distributions
  $\bernoullidist_\mu$ with means restricted to $\mu \in [a,b]$ where
  \begin{align*}
    a &= \frac{1}{1 + e^{X U/\sqrt{d}}},
    &
    b &= \frac{1}{1 + e^{-XU/\sqrt{d}}}.
  \end{align*}
  The minimax regret for the case $a=0,b=1$ is well known (see
  \citep{XieBarron2000} and references therein). We can handle general
  $a,b$ by adapting the proof of \citet{OrdentlichCover1998}. Like in
  the proof of Theorem~\ref{thm:normallocation}, the minimax regret
  equals the stochastic complexity \citep{grunwald2007minimum}:
  \begin{align*}
    \ln \sum_{z^n \in \{0,1\}^{n}} \max_{\mu \in [a,b]}
    \bernoullidist_\mu(z^n)
    &= \ln \sum_{k=0}^{n} \binom{n}{k} \max_{\mu \in [a,b]} \mu^k
    (1-\mu)^{n-k}\\
    &\geq \ln \sum_{k=\ceil{an}}^{\floor{bn}} \binom{n}{k}
    \Big(\frac{k}{n}\Big)^k
    \Big(\frac{n-k}{n}\Big)^{n-k}.
  \end{align*}
  As shown by \citet[Proof of Lemma~2]{OrdentlichCover1998}, the
  terms in this sum are at least
  \[
    \binom{n}{k} \Big(\frac{k}{n}\Big)^k \Big(\frac{n-k}{n}\Big)^{n-k}
      %\geq \frac{\Gamma(n+1)}{2^n \Gamma(\frac{n}{2}+1)^2}
      %= \frac{\Gamma(\frac{n+1}{2})}{\sqrt{\pi} \Gamma(\frac{n}{2}+1)}
      \geq \frac{1}{\sqrt{\pi (n+1)/2}}
      \qquad \text{for all $k$.}
  \]
  Hence
  \begin{align*}
    \sum_{k=\ceil{an}}^{\floor{bn}} \binom{n}{k}
    \Big(\frac{k}{n}\Big)^k
    \Big(\frac{n-k}{n}\Big)^{n-k}
    &\geq
    \frac{\floor{bn} - \ceil{an}}{\sqrt{\pi (n+1)/2}}
    \geq \frac{(b-a)n - 2}{\sqrt{\pi (n+1)/2}}
    \geq \frac{(b-a)n - 2}{\sqrt{\pi n}}\\
    &= \frac{(b-a)\sqrt{n}}{\sqrt{\pi}} - \frac{2}{\sqrt{\pi n}}.
  \end{align*}
  It remains to bound
  \[
    b-a 
       % = \frac{1}{1 + e^{-XU/\sqrt{d}}} - \frac{1}{1 + e^{XU/\sqrt{d}}}
       % = \frac{e^{XU/\sqrt{d}} - 1}{e^{XU/\sqrt{d}} + 1}
        = \tanh\Big(\frac{UX}{2\sqrt{d}}\Big)
        \geq \left(\frac{UX}{4\sqrt{d}}\right),
  \]
  where the inequality follows from $\tanh(x) \geq x/2$ for $x \in
  [0,1]$, which applies because $UX \leq 2 \sqrt{d}$ by assumption. The
  proof is completed by combining all previous steps.
\end{proof}

\interproofspace

\begin{proof}{\bfseries of Theorem~\ref{adaptivelogisticregression} \; }
  For $t < t^*$ the learner can play $p_t(y) = 1/K$, which incurs $0$
  instantaneous regret, because $\|x_t\| = 0$ implies that
  $p_{\theta,t}(y)=1/K$ for all $\theta$. Then, from $t \geq t^*$, the
  learner aggregates multiple copies of the algorithm from
  Theorem~\ref{thm:logisticupper} for $U=\alpha$, specialized to the case that
  $\Theta_U = \{\theta \in \reals^{K \times d} : \|\theta\| \leq U\}$
  and with $X = X_T$ (which is possible because the Bayesian
  algorithm described there does not depend on $X$).
  Lemma~\ref{lem:aggregate} is applied with $\alpha = \|\theta\|$, $\alphamin = \epsilon
  dK/(X_{t^*}T)$ and $\alphamax = \infty$. All together this gives the
  bound
  \begin{align*}
    R_T(\theta)
      &\leq 5 dK \ln \Big(\frac{(2\|\theta\|\bmin
              \frac{\epsilon dK}{X_{t^*} T}) X_T (T-t^*+1)}{dK} + e\Big)
        +
        2\ln \Big(\log_2\Big(\frac{8\|\theta\|X_{t^*}T}{\epsilon dK} \bmax
        4\Big)\Big)\\
      &\leq 5 dK \ln \Big(\frac{2\|\theta\| X_T T}{dK}
              + \frac{\epsilon  X_T}{X_{t^*}} + e\Big)
        +
        2\ln \Big(\log_2\Big(\frac{8\|\theta\|X_{t^*}T}{\epsilon dK} \bmax
        4\Big)\Big),
  \end{align*}
  as required. The algorithm is scale-free, because $x_t \mapsto \lambda
  x_t$ for all $t$ implies that $\alphamin \mapsto \alphamin/\lambda$,
  which is equivalent to calling Lemma~\ref{lem:aggregate} with
  $\alphamin$ unchanged but $U = \alpha/\lambda$, leading the algorithm
  from Theorem~\ref{thm:logisticupper} to produce the same predictions,
  and as a consequence $p_t$ is also unchanged.
\end{proof}

\interproofspace

\begin{proof}{\bfseries of Theorem~\ref{thm:logistic_eff_ad_sf} \; }
  As in the proof of Theorem~\ref{thm:logisticupper_efficient} we can
  predict with the uniform distribution $p_t$ for all $t < t^*$ without
  incurring instantaneous regret, so assume without loss of generality
  that $t^* = 1$.

  We then start by using a doubling trick to make the algorithm from
  Theorem~\ref{thm:logisticupper_efficient} adapt to $X$: starting from
  $X = \Xmin := \ln(K)/U$ we restart the algorithm with new value $2 X$
  any time that $\|x_t\|_2 > X$. (NB.\ If $\|x_t\|_2 > 2X$ at the time
  of a restart, we interpret it as immediately triggering more restarts
  until $X \geq \|x_t\|_2$.) Since
  \begin{align*}
    \sum_{i=0}^{\ceil{\log_2(X_T/\Xmin}) \bmax 0} \Xmin 2^i
      \leq \Xmin 2^{(\ceil{\log_2(X_T/\Xmin)} + 1) \bmax 0}
      \leq 4 X_T \bmax \Xmin
      = 4 X_T \bmax \frac{2 \ln(K)}{U},
  \end{align*}
  this leads to a regret bound of the same order as in
  Theorem~\ref{thm:logisticupper_efficient}:
  \begin{align*}
    R_T(\theta)
      &= O\Big(\sum_{i=0}^{\ceil{\log_2(X_T/\Xmin} \bmax 0}
      \big(U\Xmin 2^i +
      \ln K\big)dK\ln T\Big)\\
      &= O\Big(\big(UX_T \bmax \ln K\big)dK\ln T\Big)
      \qquad \text{for all $\theta$ such that $\|\theta\|_{2,\infty} \leq U$,}
  \end{align*}
  and runs in time $O(d^2 K^3 + (U X_t \bmax \ln K) K^2 \ln(t(1+UX_t
  \bmax \ln K)))$ in round $t$.

  Let $A(U)$ denote this algorithm. We will aggregate multiple copies of
  $A(U)$ using Lemma~\ref{lem:aggregate} with $U = \alpha =
  \|\theta\|_{2,\infty}$, $\alphamax = \frac{T^\beta}{X_{t^*}} \geq
  \frac{T^\beta}{X_T}$, $\Theta_\alpha = \{\theta \in \reals^{K \times
  d} : \|\theta\|_{2,\infty} \leq \alpha\}$, and $\alphamin =
  \frac{T^{-\gamma}}{X_{t^*}} \geq \frac{T^{-\gamma}}{X_T}$ for $\gamma
  \geq 0$ to be chosen below. We further observe that algorithms with $U
  > 2T^\beta/X_t \geq T^\beta/X_T$ will never be useful because of the
  restriction to $\|\theta\|_{2,\infty} X_T \leq T^\beta$ in
  \eqref{eqn:logistic_efficient_adaptive}, so as soon as $X_t$ becomes
  large enough for this to happen, we stop expending computation on
  $A(U)$. This can be implemented either by treating $A(U)$ as a
  sleeping expert in the sense of
  \citet{FreundEtAl1997,AdamskiyKoolenChernovVovk2016} or by simply
  setting the algorithm's predictions to the uniform distribution for
  all remaining rounds. All together, this aggregation procedure
  guarantees regret at most
  \begin{align*}
    R_T(\theta)
        = O\Big(\big(\|\theta\|_{2,\infty}X_T +
        \frac{T^{-\gamma} X_T}{X_{t^*}} + \ln K\big)dK\ln T\Big)
          + 2\ln\Big(\log_2\Big(8\|\theta\|_{2,\infty}X_T T^\gamma \bmax
          4\Big)\Big)
  \end{align*}
  for all $\theta$ such that $\|\theta\|_{2,\infty}X_T \leq T^\beta$,
  which establishes \eqref{eqn:logistic_efficient_adaptive} for $\gamma
  = \frac{c T^\beta}{d^2 K}$.

  Let us proceed to analyse the run-time. To this end, define $U_i =
  \alphamin 2^i$ and recall that algorithm $U_i$ is still running in
  round $t$ only if $U_i \leq T^\beta/X_t$. The run-time in round $t$
  therefore comes to
  \begin{align*}
    O\Big( &
      \sum_{i=0}^{\ceil{\log_2(\alphamax/\alphamin)}}
      \ind{U_i \leq T^\beta/X_t}
      \Big\{
        d^2 K^3 + (U_i X_t \bmax \ln K) K^2 \ln(t(1+U_i X_t \bmax \ln K))
      \Big\}
    \Big)
      \\
      &= 
    O\Big(
      \sum_{i=0}^{\ceil{\log_2(\alphamax/\alphamin)}}
      \Big\{
        \ind{U_i \leq T^\beta/X_t}
        U_i X_t K^2 \ln(T(1+U_i X_t))
      \Big\}\\
      &\qquad +
        \log_2\Big(\frac{\alphamax}{\alphamin}\Big)
      \Big\{
        d^2 K^3 + \ln(K) K^2 \ln(T(1+\ln K))
      \Big\}
    \Big)
      \\
    &=
    O\Big(
      \sum_{i=0}^{\argmin_j \{U_j \leq T^\beta/X_t\}}
      \Big\{
        U_i X_t K^2 \ln(T(1+T^\beta))
      \Big\}
      +
        (\beta + \gamma)
      \Big\{
        d^2 K^3 + \ln(K) K^2 \ln(T(1+\ln K))
      \Big\}
    \Big)
      \\
    &=
    O\Big(
        T^\beta K^2 \ln(T)
      +
        (\beta + \gamma)
      \Big\{
        d^2 K^3 + \ln(K) K^2 \ln(T(1+\ln K))
      \Big\}
    \Big).
  \end{align*}
  We aim to choose $\gamma$ (nearly) as large as possible to ensure that
  \[
        (\beta + \gamma)
      \Big\{
        d^2 K^3 + \ln(K) K^2 \ln(T(1+\ln K))
      \Big\} =
        O\Big(d^2 K^3 + c
        T^\beta K^2 \ln(T)\Big),
  \]
  so that the run-time per round is $O(d^2 K^3 + (1+c) T^\beta K^2 \ln(T))$. The choice
  $\gamma = \frac{c T^\beta}{d^2 K}$ indicated above satisfies
  this requirement.

  Finally, it remains to establish that the algorithm is scale-free. To
  see this, note that if we multiply all $x_t$ by some $\lambda > 0$,
  then $A(U/\lambda)$ makes the same predictions as $A(U)$ does without
  multiplication. This type of compensation is built into the
  aggregation procedure, because the definitions of $\alphamin$ and
  $\alphamax$ scale inversely with $\lambda$.
\end{proof}

\subsubsection{Besov Classes}\label{app:besov}

Let $\domX \subset \reals^d$ is compact, let $\Theta$ be the Besov space
$B_{p,q}^s(\domX)$ and identify $h_\theta \equiv \theta$ for $\theta \in
\Theta$. Suppose $\Theta_U = \{\theta \in \Theta :
\|\theta\|_{B_{p,q}^s} \leq U$ is a ball of radius $U$, where
$\|\theta\|_{B_{p,q}^s}$ is the corresponding Besov-norm.
\citet[Example~2]{foster2018logistic} show, non-constructively, that the
minimax regret is bounded by
\[
  \min_{\mathbf{Algs}} \max_\sample \max_{\theta \in \Theta_U} R_T(\theta)
    = \tilde{O}(U^\beta T^\gamma),
\]
where
\begin{enumerate}
\item If $s \geq d/2$, then $\beta = 2d/(d+2s),\gamma = d/(d+2s)$;
\item If $s < d/2$, then $\beta = 1$ and $\gamma$ depends on $p$: if
$p > 1+d/(2s)$, then $\gamma = 1-s/d$; otherwise $\gamma = 1-1/p$.
\end{enumerate}
We see that in all cases the rate depends heavily on $U$. Adaptation to
$U$ using Lemma~\ref{lem:aggregate} with $\alpha = U^\beta$, $\alphamin
= T^{-\gamma}$ and $\alphamax = \infty$ gives
\begin{theorem}\label{thm:logistic_Besov}
  Consider the Besov space setup described above for any fixed $p,q$ and
  $s$. Then there exists a learning algorithm with respect to the entire
  Besov space $B_{p,q}^s(\domX)$ that guarantees
  \[
    R_T(\theta)
      = \tilde{O}\Big( \|\theta\|_{B_{p,q}^s}^\beta T^\gamma
        + \ln(\log_2(\|\theta\|_{B_{p,q}^s}T^\gamma \bmax 1)\Big)
    \qquad \text{for all $\theta \in B_{p,q}^s(\domX)$.}
  \]
\end{theorem}
This adaptive upper bound matches the non-adaptive bound.

\section{From Log Loss to General Mixable Losses}\label{app:logtomix}
For general loss functions, the Bayesian prediction strategy from the previous 
section generalises to the Exponential Weights (EW) algorithm, which produces 
distributions $\pi_t$ that generalize the posterior distribution from \eqref{eqn:logloss_substitution} to
\[
  \der \pi_t(\theta) =
    \frac{e^{-\eta_t \sum_{s=1}^{t-1} \loss(f_\theta(x_t),y_t)} \der
    \pi(\theta)}
         {\int e^{-\eta_t \sum_{s=1}^{t-1} \loss(f_{\theta'}(x_t),y_t)} \der
    \pi(\theta')}.
\]
These depend not just on a prior $\pi$, but also on (possibly
time-varying) learning rates $\eta_t > 0$. The log loss case is
recovered for $\eta_t = 1$. Generalizing
\eqref{eqn:logloss_substitution} for the log loss, we need a way to map the
distributions $\pi_t$ over $\Theta$ to actual predictions $a_t$. To this end, let
$
  P_t = f_\theta(x_t)_\# \pi_t
$
be the distribution over predictions in $\actions$ induced by
$f_\theta(x_t)$ when the parameters $\theta$ are distributed according
to $\pi_t$ (i.e.\ the pushforward of $\pi_t$ for the map $\theta \mapsto
f_\theta(x_t)$). 
%The mappings $\theta \mapsto f_\theta(x)$ and $a_t \mapsto \ell(a, y)$ are presumed measurable for any $x \in \domX$ and any $y \in \domY$. 
Then the predictions are determined by a
\emph{substitution function} $\zeta_t$ which maps distributions on
$\actions$ to a single action:
$
  a_t = \zeta_t(P_t).
$
In case of the log loss, actions are densities and $\zeta_t(P) = \E_P[a]$
is simply the mean.

For so-called mixable loss functions $\ell$, there exists a direct
generalization of Lemma~\ref{lem:logloss}.
%(and consequently also ofLemma~\ref{lem:aggregate}). 
For $\eta > 0$, a loss function $\loss$ is
said to be \emph{$\eta$-mixable} with respect to $(\actions, \domY)$
\citep{Vovk01CompetitiveOS} if there exists a substitution
function $\zeta$ that maps any probability distribution $P$ over
$\actions$ to a single prediction $a_P = \zeta(P) \in \actions$
satisfying
\[
  \loss(a_P,y) \leq -\tfrac{1}{\eta} \ln \E_{a \sim P}\big[e^{-\eta
  \loss(a,y)}\big]
  \qquad \text{for all $y \in \domY$}.
\]
For the log loss, $1$-mixability (trivially) holds with equality when
$\zeta$ is the mean. In general, we will also cover the case that we
have prior knowledge that $y_t \in \domY_t \subseteq \domY$.
Lemma~\ref{lem:logloss} then generalizes to:
\begin{lemma}\label{lem:mixable_losses} For $t=1,\ldots,T$, suppose the
loss $\ell$ is $\eta_t$-mixable with respect to $(\actions,\domY_t)$
with $\domY_t \subseteq \domY$ for substitution function $\zeta_t$. Then
the exponential weights algorithm with non-increasing learning rates
$\eta_1 \geq \cdots \geq \eta_T > 0$ and substitution functions
$\zeta_1,\ldots,\zeta_T$ achieves
  \begin{equation}
    \sum_{t=1}^{T} \loss(a_t,y_t) \leq \E_{\theta \sim
    \gamma}\Bigg[\sum_{t=1}^T \loss(f_\theta(x_t),y_t)\Bigg]
      +\frac{\KL(\gamma\| \pi)}{\eta_T}
    \qquad
    \text{for all $\gamma$ such that $\KL(\gamma,\pi_t) < \infty$,}
  \end{equation}
  provided that the prior knowledge that $y_t \in \domY_t$ is correct..
\end{lemma}
Note that Lemma~\ref{lem:mixable_losses} specialises to any countable set of experts or continuously parameterised set of static experts.
\begin{proof}%{\bfseries of Lemma~\ref{lem:mixable_losses} \; }
	The proof is a straightforward specialisation of Lemma 1 from \cite{van-der-hoeven2018the-many}:
	\begin{lemma*}
		The FTRL version of EW with prior $\pi(\theta)$ generates a sequence of distributions $P_t$ over $\theta$ that satisfies
		\begin{equation}\label{eq:exp_weights_bound}
			\sum_{t=1}^T g_t(\theta_t) - \E_{\theta \sim \gamma}\Bigg[ \sum_{t=1}^T g_t(\theta)\Bigg]
			\leq \sum_{t=1}^{T}  \left( g_{t}(\theta_{t})+\frac{\mathrm{log}\mathbb{E}_{\theta \sim P_{t}}\left[\exp{(-\eta_tg_{t}(\theta)) }\right]}{\eta_t}
			\right)
			 + \frac{\mathrm{KL(\gamma\lVert \pi)}}{\eta_T} \, .
		\end{equation}
	\end{lemma*}
	We specify the result to $g_{t} : \theta \mapsto \ell(f_{\theta}(x_t),y)$. Then, by definition of $a_t$, and by the mixability property
	\begin{equation*}
		\ell(a_t, y_t)  + \frac{\log \E_{\theta \sim P_t} [ \exp( - \eta_t \ell(f_{\theta}((x_t)), y_t ))]}{\eta_t} \leq 0\, . 
	\end{equation*}
	Summing over $t$ and substituting in \eqref{eq:exp_weights_bound} (replacing the sum over $g_t(\theta_t)$ by a sum over $\ell(a_t, y_t)$) yields the claimed result.
\end{proof}

\section{Proofs and Additions to Sections \ref{sec:slls}}

\subsection{Proof of the Aggregation Lemma}

\begin{proof}{\bfseries of Lemma \ref{lem:agg_sq_loss} \; }
	We apply the procedure described above. Let $a_{\alpha,t}$ denote the prediction of $A(\alpha)$. Then, given any $\alpha \in (0, \alphamax)$, let $\alpha^*$ be the smallest value in the grid exceeding $\alpha$, such that
  $\alpha \leq
	\alpha^* \leq 2\alpha \bmax \alphamin
  $, and let $m^* = \log_2 \alpha^*$. 
	\begin{equation*}
		\sum_{t=1}^T \ell(a_t, y_t) - \ell(a_\theta(x_t), y_t)
		=  \sum_{t=1}^T \ell(a_t, y_t) - \ell(a_{\alpha^\star, t}, y_t) 
		+ \underbrace{\sum_{t=1}^T \ell(a_{\alpha^\star, t}, y_t) - \ell(a_{\theta}(x_t), y_t)}_{\leq B_T(\theta, \alpha^\star)}  \, .
	\end{equation*}
	Since $y_t \in \ball(0, Y_t)$, by the Pythagorean inequality, the loss $\ell(a_{\alpha^\star, t}, y_t)$ can only reduced by a projection of $a_{\alpha^\star, t}$ on $\ball(0, Y_t)$, so
	\begin{equation*}
		\sum_{t=1}^T \ell(a_t, y_t) - \ell(a_{\alpha^\star, t}, y_t)
		\leq \sum_{t=1}^T \ell(a_t, y_t) - \ell( \Pi_{Y_t}(a_{\alpha^\star, t}), y_t) \, . 
	\end{equation*}
	By clipping the losses (cf. details at the end of the proof), then applying the Pythagorean inequality again, thanks to the fact that
        \begin{equation}\label{eqn:doubleproj}
        \Pi_{Y_{t-1}}(\Pi_{Y_t}(a_{\alpha^\star, t})) = \Pi_{Y_{t-1}}(a_{\alpha^\star, t})
        \end{equation}
        and that $\tilde y_t \in \mathcal B(0, Y_{t-1})$:
	\begin{align*}
		\sum_{t=1}^T \ell(a_t, y_t) - \ell( \Pi_{Y_t}(a_{\alpha^\star, t}), y_t)
		&\leq \sum_{t=1}^T \ell(a_t, \tilde y_t) - \ell( \Pi_{Y_t}(a_{\alpha^\star, t}), \tilde y_t) + Y_T \max_{t \in [T]} \|a_t - \Pi_{Y_t}(a_{\alpha^\star, t}) \| \\
		&\leq \sum_{t=1}^T \ell(a_t, \tilde y_t) - \ell( \Pi_{Y_{t-1}}(a_{\alpha^\star, t}), \tilde y_t) + Y_T \max_{t \in [T]} \|a_t - \Pi_{Y_t}(a_{\alpha^\star, t}) \| \, . \nonumber
	\end{align*}
	We now apply Lemma~\ref{lem:mixable_losses} with $\gamma$ a point mass at $\alpha^\star$, each expert %$f_{\theta}(x_t)$
	being $\Pi_{Y_{t-1}}(a_{\alpha_j, t})$, and the losses being $\tilde y_t$, which are both in $\ball(0, Y_{t-1})$ to see that
	\begin{equation*}
		\sum_{t=1}^T \ell(a_t, \tilde y_t) - \ell( \Pi_{Y_{t-1}}(a_{\alpha^\star, t}), \tilde y_t)
		\leq \frac{-\log \pi(m^\star)}{\eta_{T}} =  4Y_{T-1}^2 \big( - \log \pi(m^\star) \big)\, .
	\end{equation*}
	Then, since the substitution function we use for the aggregation is the mean, $a_t$ is a convex combination of the $\Pi_{Y_{t-1}}(a_{\alpha_j, t})$'s, so $\|a_t\| \leq Y_{t-1}$, and $\| a_t - \Pi_{Y_t}(a_{\alpha^\star, t})  \| \leq Y_{t-1} + Y_t \leq 2Y_t$.
  Therefore for any $\theta \in \Theta$%\Theta_{\alpha} \subset \Theta_{\alpha^\star}$
	\begin{equation*}
		R_T(\theta) 
    \leq B(T, \theta, \alpha^\star) - 4 Y_{T-1}^2 \log \pi(m^\star)
    + 2Y^2  
		\leq \max_{\alpha' \in [\alpha, 2\alpha \bmax \alphamin]} B_T(\theta,\alpha') - 4Y_{T-1}^2 \log \pi(m^\star)  + 2Y_T^2
	\end{equation*}
	Then observe that
	\begin{equation*}
		 -\ln \pi(m^\star)
		< 2\ln (m^\star+2)
		 \leq  2\ln (\log_2(\alpha^\star) +2)
		\leq  2\ln (\log_2(8(\alpha/\alphamin) \bmax 4))
	\end{equation*}
	to conclude.

  \paragraph{Details on the Clipping} \newcommand{\aclip}{\tilde a_{t, \alpha^\star}}
  Clipping works similarly to the linear case, because regret depends affinely on the data $y_t$. Indeed, by expanding the squares,
  \begin{equation*}
    \|a_t - y\|^2  - \|\aclip - y\|^2 
    = \| a_t- \tilde y_t\|^2 - \|\aclip - \tilde y_t\|^2  - 2 \langle  y_t - \tilde y_t, a_t- \aclip \rangle \, ,
  \end{equation*}
  where we denoted $\aclip = \Pi_{Y_t}(a_{t, \alpha^\star})$ to reduce clutter.
  The linear overhead can be bounded by Cauchy-Schwarz, $|\langle  y_t - \tilde y_t , \, a_t - \aclip \rangle|\leq  \|y_t- \tilde y_t\|\| a_t - \aclip\|$, and 
  \begin{equation*}
    \| y_t - \tilde y_t \|= \|y_t \| \, |1 - Y_{t-1} / Y_t | \leq (Y_t - Y_{t-1}) \, . 
  \end{equation*}
  Therefore, by summing over $t$, upper bounding $\|a_t- \aclip\|$ by its maximum over $t$, simplifying the telescoping sum and dividing by $2$ to recover the square loss, 
  \begin{equation*}
    \sum_{t=1}^T \ell(a_t, y_t) - \ell(\aclip, y_t)
    \leq \sum_{t=1}^T \ell(a_t, y_t) - \ell(\aclip, y_t) 
    + Y_T \max_{t \in [T]}\|a_t - \aclip\|.
  \end{equation*}

  \paragraph{Scale-invariance}
  Scale-freeness with respect to the features is straightforward, as the aggregation procedure does not look at the features. 

  Let us prove the scale-invariance with respect to the data points. If all $y_t$'s are multiplied by a factor $a$, then the actions returned by the experts are multiplied by $a$. The clipping threshold is multiplied by $a$, and thus both the clipped actions and the clipped data points are multiplied by $a$. Therefore the losses fed to the aggregation procedure is multiplied by $a^2$. 
  The learning rate in the aggregation procedure is multiplied by $a^2$. Therefore the mass put on every expert is kept the same. Since the output of every expert was multiplied by $a$, the final action is also multiplied by $a$.
\end{proof}

\subsection{Details for the Square Loss}

\subsubsection{Gradient Descent}\label{app:gradient-descent}

For any $\lambda > 0$, Gradient Descent (GD) tuned with step size $\eta_t = 1 /( \lambda +  t)$ on square losses is equivalent to %Follow-the-Regularised leader with regulariser $\lambda \|\cdot \|^2$, and to 
Exponential Weights with learning rate $1/\lambda$ and Gaussian prior $\Sigma = I_d$ with the mean as a substitution function; this was observed by \cite{koolen2016exploiting} and \citet{van-der-hoeven2018the-many}. We recall a slightly modified version of Corollary~6 in the latter reference.
\begin{theorem}
	For prediction with the square loss, for any $\alpha > 0$, gradient descent with step size $\eta_t = 1  / (\lambda + t)$ is scale-free and enjoys the regret bound
	\begin{equation}\label{eq:gd_sq_loss}
		R_T (\theta) 
		\leq \frac{\lambda\|\theta\|^2}{2} + 2 Y_T^2 \log \bigg(1 +  \frac{T}{\lambda} \bigg) \, , \quad \text{for any $\theta \in \R^d$.}
	\end{equation}
	Furthermore, the updates are such that $\| \theta_t\| \leq Y_{t-1}$ for all $t$.
\end{theorem}
The analysis can be made tighter so that the bound does not diverge when $\alpha \to \infty$, but we chose the bound simplest to read.

\subsubsection{KAAR}
Let us recall the guarantees for scale-free \algname{KAAR}, from \citet[Theorem~2]{gammerman2004on-line}. We slightly adapt the statement to include the scaling of the regularisation by $X_{t^\star}$.
Note that if $k(x_t, x_t) = 0$, then $k(x_t, \cdot) = 0$ and $\theta(x_t) = 0$ for any $x_t$. Thefore predicting $a_t = 0$ on all the rounds for which $k(x_t, x_t) = 0$ has no impact on the regret. 

The formula for the updates of \algname{KAAR} gives the updates of \algname{KAAR-sf} for all $t\geq t^\star$ and are still given by \eqref{eq:kaar-updates}. The scale-free property derives directly from the updates formula. 
\begin{theorem}
	\algname{KAAR-sf}$(\alpha)$ over the RKHS $\mathcal F$ is scale-free and guarantees that for any $\theta \in \mathcal F$.
	\begin{equation}\label{eq:kaar-regret-bound}
		R_T(\theta) 
		\leq 
		\frac{  \alpha X_{t^\star}\|\theta\|^2}{2} + \frac{Y_T^2}{2}\ln \det \bigg( I_T +  \frac{1}{\alpha X_{t^\star}} K_T \bigg) \,  ,
	\end{equation}
	where $K_T = \big(k(x_u, x_v)\big)_{(u, v)\in [T]^2}$ and $t^\star = \min \{t \, | \,  k(x_t, x_t) > 0 \}$. 
\end{theorem}
We use the convention that $1/(X_{t^\star}) K_T = 0$ if $T \leq t^\star$. (In this case, all features up to time $t$ are $0$, the algorithm predicts only $0$ and the regret at time $T$ is exactly $0$.)

\subsection{Proofs of the Regret Bounds}
\interproofspace
\begin{proof}{\bfseries of Theorem~\ref{thm:sq_loss_full_bound} \; }
  We prove the result for arbitrary values of $\alphamin$ and $\alphamax$, then specialize to $\alphamin=1$ and $\alphamax = T$. Define $Y = Y_T$ to simplify notation. First, assume that $\|\theta\| \leq Y$.
	Plug in the value $\alpha = (Y^2 / \| \theta\|^2) \bmin \alphamax  $ in the upper bound
	\begin{align*}
		R_T(\theta) 
		&\leq \max_{\alpha' \in [\alpha, \, 2\alpha \bmax \alphamin]} \Bigg\{ \frac{\alpha'\| \theta\|^2}{2} 
		+ 2 Y^2 \log \big(1 +  \frac{1}{\alpha'} T \big) \Bigg\}
		+ \frac{1}{\eta_{T-1}}\ln \big(\log_2((\frac{8 \alpha}{\alphamin} \bmax 4\alphamin)\big)+2Y^2 \\
		&\leq \frac{(2 \alpha \bmax \alphamin) \| \theta\|^2}{2} + 2Y^2 \log \bigg(  1 + \frac{1}{\alpha} T \bigg) 
		+ 8Y^2\ln \big(\log_2((8 \alpha/\alphamin) \bmax 4)\big) +2Y^2\\
		&\leq Y^2 \bmax \frac{\alphamin\|\theta\|^2}{2} 
    + 2Y^2 \log \bigg(  1+ \Big(\frac{\|\theta\|^2}{Y^2}\bmax \frac{1}{\alphamax} \Big) T \bigg) + 8Y^2\ln \big(\log_2(\frac{8Y^2}{\alphamin\|\theta\|^2} \bmax 4)\big) +2Y^2 \, .
	\end{align*}
  Note that as $\alphamin \leq 1$ and $(\|\theta\| / Y)^2 \leq 1$ we have
  $
    Y^2 \bmax (\alphamin\|\theta\|^2 / 2) =  Y^2 \, .
  $
  The claimed bound follows after replacing $\alphamin = 1$ and $\alphamax = T$, applying the bound $a \bmax b \leq a+b$ for $a, b >0$.

	If $\|\theta \| > Y$ consider $\tilde \theta = \Pi_{Y}(\theta)$. Then by the Pythagorean inequality
	$
	R_T(\theta)
		\leq R_T(\tilde \theta) \, .
	$
\end{proof}

\interproofspace

\begin{proof}{\bfseries of Theorem~\ref{thm:akaar-general} \; }
	We build an algorithm via a double infinite aggregation procedure. Define the algorithm $A(\alpha)$ to be \algname{KAAR-sf}$(\alpha)$, and define its regret bound:
	\begin{equation*}
		B_T(\alpha, \theta) := \frac{  \alpha X_{t^\star}^2\|\theta\|^2}{2} + \frac{Y^2}{2}\ln \det \bigg( I_T +  \frac{1}{\alpha X_{t^\star}} K_T \bigg) \, ;
	\end{equation*}
	with the convention that $B_T(\alpha, \theta) = 0$ if $t \leq t^\star$.
	For any $c > 0$, define $\tilde A(c)$ to be result of the aggregation procedure applied to $A(\alpha)$ tuned with $\alphamin = 1 / c$, and $\alphamax = \infty$. Since each instance is scale-free, the aggregated version is also scale-free, by Lemma~\ref{lem:agg_sq_loss}. The algorithm $\tilde A(c)$ enjoys the regret bound
	\begin{multline*}
		R_T(\theta) \leq \max_{\alpha' \in [\alpha,2\alpha \bmax (1/c)]} B_T(\theta,\alpha') + 8Y^2 \ln (\log_2((8\alpha c) \bmax 4)) 
		+ 2Y^2  \\
		\leq \underbrace{\Big((2 \alpha)\bmax \frac{1}{c}\Big)\frac{X_{t^\star}^2\|\theta\|^2}{2}
		+ \frac{Y^2}{2} \ln \det \bigg( I_T + \frac{1}{\alpha X_{t^\star}^2}  K_T \bigg)
		+ 8Y^2 \ln (\log_2((8\alpha c) \bmax 4)) 
		+ 2Y^2}_{:=\tilde B_T(\theta, c)} \, .
	\end{multline*}
	Now run the aggregation procedure again, with each expert being $\tilde A(c)$, this time with the parameters $c_{\min} = 1$ and $c_{\max} = \infty$. Again, Lemma~\ref{lem:agg_sq_loss} guarantees that the total algorithm is also scale-free. Then for any $c \geq 0$
	\begin{align*}
		R_T(\theta)
		&\leq \max_{c' \in [c, 2c \bmax 1]} \tilde B_T(c, \theta)
		+ 2Y^2 + 8Y^2 \ln \log_2( 8c \bmax 4)  \\
		&\leq 
		\Big((2 \alpha)\bmax \frac{1}{c} \Big)\frac{X_{t^\star}^2\|\theta\|^2}{2}
		+ 8Y^2 \ln \log_2((8\alpha (2c\bmax 1)) \bmax 4) 
		 \\
		 &+ \frac{Y^2}{2} \ln \det \bigg( I_T + \frac{1}{\alpha X_{t^\star}^2} K_T \bigg) + 2Y^2
		+ 2Y^2 + 8Y^2 \ln \log_2( 8c \bmax 4) \, .
	\end{align*}
	In particular, for $c = 1 / (2\alpha)$, noting that $8\alpha(2c \bmax 1) = 8 \bmax 8\alpha$
	\begin{equation*}
		R_T(\theta)
		\leq \frac{\alpha X_{t^\star}^2\|\theta\|^2}{2} 
		+ \frac{Y^2}{2} \ln \det \bigg( I_T + \frac{1}{\alpha X_{t^\star}^2}  K_T \bigg)
		+ 4Y^2
		+ 8Y^2 \ln \log \left(8 \alpha \bmax 8\right)
		+ 8Y^2 \ln \log \left(\frac{4}{\alpha} \bmax 4   \right) \,.
	\end{equation*}
	Finally, upper bounding $4$ by $8$ inside the logarithms and using a case disjunction on whether $\alpha \geq 1$,
	\begin{multline*}
		\ln \log_2 (8\alpha\bmax 8)  + \ln \log_2 (4\alpha^{-1}\bmax 4)
		\leq \ln \log_2 (8\alpha\bmax 8)  + \ln \log_2 (8\alpha^{-1}\bmax 8) \\
		= \ln \log_2 8 + \ln \log_2 (\alpha \bmax \alpha^{-1} )
		=  \ln \big( 3 | \log_2 \alpha | \big) \, .
	\end{multline*}
	Reparameterize by $\lambda = \alpha / X_{t^\star}^2$ to obtain the regret bound.
\end{proof}
\subsection{Consequences of the General Regret Bound}

\subsubsection{Dimension-independent Bound}
Note that regardless of the kernel, denoting by $\lambda_n (K_T)$ the $n$-the largest eigenvalue of $K_T$,
\begin{equation}\label{eq:effective-dimension}
	\ln \det \bigg( I_T + \frac{1}{\lambda} K_T \bigg) 
	= \sum_{n = 1}^T \ln \bigg( 1 + \frac{ \lambda_n (K_T)}{\lambda} \bigg) 
	%\leq \alpha\sum_{n=1}^T \lambda_n(K_T) 
	\leq \frac{\Tr(K_T) }{\lambda} 
	%= \alpha \sum_{n = 1}^T k(x_t, x_t)
	\leq \frac{TX_T^2}{\lambda}  \, .
\end{equation}
In particular, after applying this upper bound, optimizing \eqref{eq:a-kaar-general} over $\lambda$ to get $\lambda = X_TY_T\sqrt T / \|\theta\|$, we see that \algname{A-KAAR} enjoys the dimension-independent bound Corollary~\ref{cor:dim-independent-KAAR}.

\subsubsection{Parametric Case}
Using concavity of the logarithm,
\begin{equation*}
	\ln \det \bigg(I_T + \frac{1}{\lambda} K_T \bigg)
	= \sum_{n = 1}^d \ln \biggl( 1 +  \frac{\lambda_n(K_T)}{\lambda} \biggr) 
	%\leq d \ln \biggl(  1 + \frac{\lambda}{d}  \sum_{n=1}^d \lambda_n(K_T)  \biggr)
	\leq  d \ln \biggl(  1 + \frac{ \Tr(K_T)}{\lambda d}   \biggr)
	\leq  d \ln \biggl( 1 + \frac{ T X_T^2}{\lambda d } \biggr)
	\, . 
\end{equation*}
Apply the inequality above in \eqref{eq:a-kaar-general} and plug in the value
	$
		\lambda =  (dY^2)/ \|\theta\|^2 
	$
to obtain Theorem~\ref{thm:regression-finite-dim}.

\interproofspace

\subsection{RKHS with the Capacity Condition}\label{app:rkhs}

The effective dimension of the kernel matrix $K_T$ at scale $\lambda$ is defined as
\begin{equation*}
  d_{\mathrm{eff}}(\lambda) = \Tr\big(K_T(K_T + \lambda I_T)^{-1}\big) \, .
\end{equation*}
It is a quantity that appears naturally in the analysis of kernel ridge regression, a widely studied variant of KAAR in the batch version of the problem. To analyse KAAR, \citet[Proposition~2]{jezequel2019efficient} prove that
\begin{equation}\label{eq:capacity-condition}
	\ln \det \biggl( I_T  + \frac{1}{\lambda} K_T \biggr)
	\leq d_{\mathrm{eff}}(\lambda)\biggl( 1 + \log \bigg(1 + \frac{ T X_T^2}{\lambda} \bigg) \biggr) \, .
\end{equation}
The capacity condition then provides an upper bound on the term above, which yields and explicit regret bound when used in \eqref{eq:a-kaar-general}. We plug in the value of $\lambda$ that optimises this regret bound.
\interproofspace
\begin{proof}{\bfseries of Theorem \ref{thm:agg-KAAR-RKHS} \; }
	Plug in the order optimal value in the A-KAAR upper bound \eqref{eq:a-kaar-general}
	\begin{equation*}
		\lambda = \left( \frac{Y^{2} (TC_k)^{\gamma}}{ \|\theta\|^{2}} \right)^{1 / (1 + \gamma)} \, , 
	\end{equation*}
	which roughly balances the two terms (that is, up to the logarithms). Then the first term in \eqref{eq:a-kaar-general} becomes
	\begin{equation*}
		\frac{\lambda \|\theta\|^2}{2} = \frac{1}{2} 
		Y^{2/ (1 + \gamma)} \|\theta\|^{2\gamma / (1+ \gamma)} (C_k T)^{2\gamma / (1+ \gamma)} \, ,
	\end{equation*}
	and, after applying \eqref{eq:capacity-condition}, using the capacity condition and replacing $\lambda$ by its value
	\begin{multline*}
		Y^2\ln \det \bigg(I_T +  \frac{1}{\lambda} K_T \bigg)  \leq \bigg( \frac{T C_k}{\lambda} \bigg)^\gamma \biggl( 1 + \log \bigg(1 + \frac{ T X_T^2}{\lambda} \bigg) \biggr) \\
		\leq  
		2^\gamma Y^{2/ (1 + \gamma)} \|\theta\|^{2\gamma / (1+ \gamma)} (C_k T)^{2\gamma / (1+ \gamma)}  
		\Bigg( 1 + 
		\log \bigg(1 + X^2\bigg(\frac{\|\theta\|^2T^\gamma}{Y^2C_k^\gamma}\bigg)^{1/(1+\gamma)} \bigg)\Bigg)
		\end{multline*}
	to obtain the final regret bound
	\begin{multline}\label{eq:full-kernel-adapt-bound}
		R_T(\theta) \leq 
			 Y^{2/ (1 + \gamma)} \|\theta\|^{2\gamma / (1+ \gamma)} (C_k T)^{2\gamma / (1+ \gamma)}  
			 \Bigg( \frac{1}{2} + 2^{\gamma-1} +
		\log \bigg(1 + X^2\bigg(\frac{\|\theta\|^2T^\gamma}{Y^2C_k^\gamma}\bigg)^{1/(1+\gamma)} \bigg)\Bigg) \\
		+ 8Y^2 \ln \Bigg( \frac{3e}{2(1+\gamma)} \bigg| \log_2 \bigg( \frac{Y^{2} (TC_k)^{\gamma}}{ \|\theta\|^{2} X_{t^\star}^{2+2\gamma}} \bigg) \bigg| \Bigg) \,,
	\end{multline}
  which is the finite-time version of the claimed result.
\end{proof}

\interproofspace

\begin{proof}{\bfseries of Corollary~\ref{thm:kaar-sobolev} \;}
  By \citet[Theorem~3]{zadorozhnyi2021online}, $\mathcal F = W_{s, 2}([-1, 1]^d)$ is an RKHS that satisfies the capacity condition with $\gamma = d / (2s)$ for some $C_k>0$, which depends on $d$ and $s$. By playing according to \algname{A-KAAR} on $\mathcal F $, Theorem~\ref{thm:agg-KAAR-RKHS} gives the claimed regret bound against any comparator $\theta \in \mathcal F = W_{s, 2}([-1, 1]^2)$, yielding the result for the $p = 2$ case.  

  For $p \geq 2$, by standard $L_p$-inclusions, $W_{s, p}([-1, 1]^2) \subset W_{s, 2}([-1, 1]^2)$, and for any element $f \in W_{s, p}([-1, 1]^2)$, denoting by $\|\cdot\|_{p}$ the $p$-norm with respect to the Lebesgue measure, and by $D^\alpha$ the (weak) partial differential of order $\alpha$, 
  \begin{multline*}
    \|f\|_{s, p} = \biggl( \sum_{|\alpha| \leq s} \|D^\alpha f\|_p^p \biggr)^{1/p}
    \leq \bigg(\sum_{|\alpha| \leq s} 1 \bigg)^{1/2 - 1/p} \biggl( \sum_{|\alpha| \leq s} \|D^\alpha f\|_p^2 \biggr)^{1/2} \\
    \leq  \big(2^s \Vol(\domX)\big)^{1/2 - 1/p} \biggl( \sum_{|\alpha| \leq s} \|D^\alpha f\|_2^2 \biggr)^{1/2} 
    =  \big(2^s \Vol(\domX)\big)^{1/2 - 1/p} \|f\|_{s, 2} \,. 
  \end{multline*}
  Therefore the result for $p \geq 2$ follows from the $p=2$ case.
\end{proof}
\interproofspace
\subsection{Lower Bound for the Square Loss}\label{app:lower-bounds-sq}
We separate the proof into three statements, considering different
parameter regimes; in particular, even though the bounds of
Propositions~\ref{thm:sq_lower_boundII} and \ref{thm:sq_lower_boundIII}
are of the same order, we separate them since the proofs are different.
\begin{proposition}\label{thm:sq_lower_boundI}
  In linear least-squares regression in $\R^d$, for any values $Y, U, X> 0$ that satisfy $(UX / Y)^2 \leq d $, for any algorithm, there exists a sequence of examples $(x_t, y_t) \in \ball(0, X) \times \ball(0, Y)$ such that
	\begin{equation*}
		\sup_{\theta\in\ball(0,U)}R_{T}(\theta)\geq 0.36 \, dY^2 \log \biggl( \Big\lfloor \frac{T}{d}  \Big\rfloor \frac{(UX/2Y)^2}{2 d\log(2\lfloor T/d \rfloor)}  + 1\biggr) - 4dY^2 \, .
	\end{equation*}
	The constant $0.36$ in the bound can be replaced by a $T$-dependent quantity that converges to $1/2$ as $T \to \infty$ and other parameters are kept constant, cf. \eqref{eq:vovk-lower-bound-T-dep}. Moreover, if $d = 1$, the same bound holds with $x_t = X$ for all $t$.
\end{proposition}
\begin{proposition}\label{thm:sq_lower_boundII}
	In linear least-squares regression in $\R^d$, for any values $Y, U, X> 0$ that satisfy $d \leq T \leq (d/8)(Y/(XU))^2$, for any algorithm, there exists a sequence of examples $(x_t, y_t)$ in $\ball(0, X) \times [-Y, Y]$ such that
	\begin{equation*}
		\sup_{\theta \in \ball(0, U)} R_T(\theta) \geq \frac{\sqrt{2}}{8} \min(UX,\, Y)Y \sqrt{T} \, .
	\end{equation*}
\end{proposition}
\begin{proposition}\label{thm:sq_lower_boundIII}
	In linear least-squares regression in $\R^d$, for any values  $Y, U, X> 0$ such that $T \leq d$, for any algorithm, there exists a sequence of examples $(x_t, y_t)$ in $\ball(0, X) \times [-Y, Y]$, such that
	\begin{equation*}
		\sup_{\theta \in \ball(0, U)} R_T(\theta) \geq \frac{1}{2}\min(UX,\, Y)Y\sqrt T \, .
	\end{equation*}
\end{proposition}

\begin{proof}{\bfseries of Theorem~\ref{thm:sq_lower_boundI} \; }
The proof follows from an alteration of the proof of Theorem 2 from \citet{Vovk01CompetitiveOS}. We reproduce it in detail for completeness. We start with the case $d = 1$ and $x_t = 1$ for all times $t$, and leave the general case for later.

A standard method in lower bounds for online learning is to build a distribution over $y_t$, and to lower bound the regret on average according to that distribution. Let $y_{t}\in \{0,1\}$ be i.i.d. Bernoulli random variables with parameter $p\in[0,1]$, itself drawn from a (symmetric) beta distribution with parameters $(A,A)$. 

Denote by $\E$ the expectation with respect to the whole randomness, that is, both the prior $\pi$ and the distribution of the $y_t$'s. The natural comparator is $p$ in this construction, and the quantity $\E[R_T(p)]$ can be explicitly lower bounded, as we shall see in \eqref{eq:bayesian_regret}. Since the distribution $\beta(A, A)$ puts mass on values of $p$ outside of the set of comparators $\ball(1/2, U)$, we lower bound the worst-case regret as
\begin{align*}
	\sup_{(y_t)_{t \in [T]}} \sup_{\theta \in \ball(1/2, U)} R_T(\theta)
	&\geq \E \bigg[ \sup_{\theta \in \ball(1/2, U)} R_T(\theta) \bigg]
	\\
	& \geq \E \bigg[ \sup_{\theta \in \ball(1/2, U)} R_T(\theta) \, \mathbf{1}\{ p \in \ball(1/2, U) \} \bigg]  \\
	& \geq \E \Big[ R_T(p) \mathbf{1}\{ p \in \ball(1/2, U) \} \Big] \\
	& =  \E \big[ R_T(p) \big] -  \E \Big[ R_T(p) \mathbf{1}\{ p \notin \ball(1/2, U) \} \Big] \\
	& \geq  \E \big[ R_T(p) \big] - T \pi\big( [0,1] \setminus \ball(1/2, U)\big) \, . 
\end{align*}
Let us now bound these two terms separately. 
\paragraph{Probability of $p$ Being Outside $B(1/2, U)$}
Since the $\beta(A,A)$ distribution is subgaussian with subgaussianity constant $1 / (4(2A+1))$ (see, e.g., \cite{marchalarbel17}), the Chernoff bound holds:
\begin{equation*}
\mathbb{P}_{p\sim \beta(A,A)}\left[\bigg|\frac{1}{2}-p\bigg|\geq U\right]\leq 2e^{-2U^{2}(2A+1)}.
\end{equation*}
Picking $A \geq \frac{\mathrm{log}\left(2T\right)}{4U^{2}}-\frac{1}{2}$ guarantees that
\begin{equation*}
	\pi\big( [0,1] \setminus \ball(1/2, U)\big) \leq \frac{1}{T}.
\end{equation*}
\paragraph{Bayesian Regret} Consider the expected value of the regret against the comparator $p$.
\begin{equation*}
	\E[R_T(p)] = \E\Bigg[ \sum_{t= 1}^T (y_t - a_t)^2 - (y_t - p)^2\Bigg]
	= \E \bigg[ \sum_{t=1}^T (y_t - a_t)(p - a_t) + (y_t - p)(p - x_t)  \bigg] \, . 
\end{equation*}
Now the law of $y_t$ given $p$ is $\Ber(p)$, and $y_t$ given $p$ is independent from $x_t$. So upon conditioning over $p$ and applying the tower rule, we get
\begin{equation*}
	\E[R_T(p)] = \E\Bigg[\sum_{t=1}^T (a_t - p)^2 \Bigg] \, .
\end{equation*}
Moreover, note that since $a_t$ is $\sigma(y_1, \dots, y_{t-1})$-measurable, 
\begin{equation*}
	\E\big[(a_t - p)^2 \mid y_1, \dots , y_{t-1}\big]
	\geq \E\Big[ \bigl( \E\,[p \mid y_1, \dots, y_{t-1}]) - p \bigr)^2  \mid y_1, \dots, y_{t-1}  \Big] \, . 
\end{equation*}
(One can interpret this as saying that if the player knows in advance that the adversary will pick $p$ according to $\pi$ and then generate $y_t$'s iid, then the best the player can do is play the expected value of $p$ given the observations.)
Denote by $R_T^\star(p)$ the regret against the value of $p$ of the optimal strategy playing $a_t^\star := \E\,[p \mid y_1, \dots, y_{t-1}]$ at every time step; then for any strategy of the learner, $\E[R_T(p)] \geq \E[R_T^\star(p)]$.

For the specific choice of prior $\pi = \beta(A, A)$, \citet{Vovk01CompetitiveOS} computes the nice closed-form expression for $a^\star_t$, namely, 
\begin{equation*}
	a^\star_t = \frac{\sum_{s= 1}^{t-1}y_s + A}{t-1 + 2A}\, .
\end{equation*}
Now the whole expected regret is amenable to computation. Indeed, conditionally on $p$
\begin{align*}
	\E[(a^\star_{t+1} - p)^2 \vert p]
	&= \frac{1}{(t + 2A)^2} \E\left[  \left(\sum_{s=1}^t  y_t + A - pt - 2Ap \right)^2 \bigg \vert \;  p\right] \\
	&= \frac{1}{(t + 2A)^2} \E\left[  \left(\sum_{s=1}^t  (y_t - p) + A(1 - 2p) \right)^2 \bigg \vert \;  p \right] \\
	&= \frac{1}{(t + 2A)^2} \left(\sum_{s=1}^t\E\left[    (y_t - p)^2 \vert p \right]+ A^2(1 - 2p)^2  \right) \\
	&= \frac{1}{(t + 2A)^2} \left( tp(1-p) + A^2(1 -2p)^2  \right)   \geq p(1-p) \frac{t}{(t + 2A)^2} \, ,
\end{align*}
where we used the fact that the variables $(y_t - p)$ are independent and centered conditionally on $p$, and that their variance given $p$ is $p(1-p)$.
Therefore, 
\begin{equation*}
	\E\big[R_T^\star(p) \big\vert  \; p \big] \geq p(1-p)\sum_{t=1}^T \frac{t-1}{(t-1 + 2A)^2} .
\end{equation*}
Let us lower bound the sum by comparing it to an integral
\begin{equation*}
	\sum_{t=1}^{T}\frac{t-1}{(t-1+2A)^{2}}\geq \int_{0}^{T}\frac{u}{(u+2A)^{2}}\mathrm{d}u 
	=\frac{1}{2} \mathrm{log}\left(\frac{T}{2A}+1\right),
\end{equation*}
Finally, averaging over the prior distribution gives that $\mathbb{E}_{p \sim \beta(A,A)}\left[p(1-p)\right]=\frac{A}{4A+2}$, and consequently, for any strategy of the learner
\begin{equation}\label{eq:bayesian_regret}
	\E[R_T(p)] \geq \E[R_T(p)^\star] \geq \frac{A}{2(4A + 2)} \log \bigg(\frac{T}{2A} + 1 \bigg) 
	\, .
\end{equation}
\paragraph{Concluding the $1$-Dimensional Case}
We have shown that for $A \geq \log (2T) / (4U^2) -1/2$, and for any strategy, the worst-case regret against the comparator set $\ball(1/2, U)$ is lower bounded by
\begin{equation*}
	\frac{A}{2(4A + 2)} \log \bigg(\frac{T}{2A} + 1 \bigg)  - 1 \,. 
\end{equation*}
So replacing $A = \log (2T) / U^2$ 
\begin{multline}\label{eq:vovk-lower-bound-T-dep}
	\sup_{\theta\in\ball(1/2,U)}R_{T}(\theta) 
	\geq \frac{\log(2T) / U^2}{2(4 \log(2T) / U^2 + 2)} \log \biggl( \frac{TU^2}{2 \log(2T)}  + 1\biggr) -1 \\
	\geq \frac{\log(2)/2}{2(4 \log(2) / 2  + 2)} \log \biggl( \frac{TU^2}{2 \log(2T)}  + 1\biggr) -1  
	\geq 0.09 \log \biggl( \frac{TU^2}{2 \log(2T)}  + 1\biggr) -1 \,.
\end{multline}
(We used the fact that $a \mapsto a / (4a + a)$ is decreasing, and we bounded $U \leq 1/2$ and $T \geq 1$.)
\paragraph{Scaling}
To generalize to an adversary playing in $y_t \in [-Y, Y]$ and comparators in $[-U, U]$, note that $y_t/ (2Y) + 1/2 \in [0,1]$, so 
\begin{multline*}
	\sup_{(y_t)\in [-Y, Y]^T} \sup_{\theta\in\ball(0,U)}R_{T}(\theta) 
	= (2Y)^2 \sup_{(y_t)\in [0, 1]^T} \sup_{\theta\in\ball(1/2,U/(2Y))}R_{T}(\theta) \\
	\geq 0.36 \, Y^2 \log \biggl( \frac{T(U/2Y)^2}{2 \log(2T)}  + 1\biggr) - 4Y^2  \,. 
\end{multline*}
\paragraph{Generalizing to $d$-Dimensional Regression}
As in \citet{Vovk01CompetitiveOS}, consider the sequence of features $x_1 = (X, 0, \dots, 0)$, $x_2 = (0, X, 0,  \dots, 0)$, etc. Then partition the time steps according to the feature values. The regret over the $d$ partitions of $\lfloor T / d \rfloor$ time steps is then lower bounded as
\begin{equation*}
	\sup_{(y_t)\in [-Y, Y]^T}  \sup_{\theta \in [-U, U]^d} R_T(\theta) \geq
	0.36 \, dY^2 \log \biggl( \Big\lfloor\frac{T}{d}\Big\rfloor \frac{(UX/2Y)^2}{2 \log(2T/d)}  + 1\biggr) - 4dY^2 \,. 
\end{equation*}
Note that the best comparator could \emph{a priori} be anywhere in $[-U, U]^d \subset \ball\big(0, \sqrt{d} U\big)$. We rescale the value of $U$ by $\sqrt d$ to obtain the claimed result.
\end{proof}

\interproofspace

\begin{proof}{\bfseries of Theorem~\ref{thm:sq_lower_boundII}\; }
  Again, we start in the $1$-dimensional case and consider a constant sequence of features $x_t = 1$.
	Let $(y_t)$ be a sequence of i.i.d. random variables that take values $-Y$ or $Y$ with probability $1/2$, and set $x_t = 1$ for all $t$. The general case follows from rescaling $\theta \leftarrow X \theta$. We only consider the comparators $-U$ and $U$, then
	\begin{align*}
		\E\!\bigg[ \max_{\theta \in [-U, U]} R_T(\theta) \bigg]& \geq \E\bigg[\max_{\theta \in \{-U, U\}} \bigg\{ \sum_{t=1}^T (\theta_t -y_t)^2 - (\theta - y_t)^2 \bigg\} \bigg] \\
		& = \E\!\bigg[\max_{\theta \in \{-U, U\}} \bigg\{ \sum_{t=1}^T \theta_t^2 - \theta^2 + 2y_t\theta - 2y_t\theta_t \bigg\} \bigg] \\
		& =  \E\!\bigg[  \sum_{t=1}^T \big(\theta_t^2 - U^2 -2 y_t\theta_t) + U\Big| \sum_{t=1}^T y_t  \Big|  \bigg] \\
		& = \E\!\bigg[ \sum_{t=1}^T \theta_t^2\bigg] - TU^2  + U\E\!\bigg[\Big| \sum_{t=1}^T y_t  \Big|  \bigg] \geq U\E\!\bigg[\Big| \sum_{t=1}^T y_t  \Big|  \bigg] - T U^2 \, .
	\end{align*}
	We used the fact that $\E[y_t \theta_t] = \E[y_t]\,\E[\theta_t] = 0$, since $y_t$ is independent from $\theta_t$. Now since each $y_t$ is either $-Y$ or $Y$ with probability $1/2$, by Lemma A.9. in \citet{cesa-bianchi2006prediction}
	\begin{equation*}
		\E\bigg[\Big| \sum_{t=1}^T y_t  \Big|  \bigg]
		\geq Y\sqrt{\frac{T}{2}}
		\, .
	\end{equation*}
	Then
	\begin{equation*}
		\max_{y_{1:T} \in [-Y, Y]^T} \max_{\theta \in [-U, U]} R_T(\theta)
		\geq
		\E\!\bigg[ \max_{\theta \in [-U, U]} R_T(\theta)  \bigg]
		\geq U Y  \sqrt{\frac{T}{2}} \bigg(1 - \sqrt{2T} \frac{U}{Y} \bigg) \, .
	\end{equation*}
	The claimed bound follows by plugging in the condition that $T \leq (1/8)(Y/(XU))^2$.

  \paragraph{Extending to Dimension $d$}
  Using the sequence of feature $x_t = X e_{(t \mod d) + 1}$ and partitioning the time steps depending on the feature value, for any $T \leq (1/8)(Y/(XU))^2$,
  \begin{equation*}
    \sup_{(y_t)\in [-Y, Y]^T}  \sup_{\theta \in [-U, U]^d} R_T(\theta) \geq
    \frac{\sqrt{2}}{4} UXY \sum_{i = 1}^d\sqrt{T_i}
  \end{equation*}
  where $T_i = \lfloor T / d \rfloor + \mathds 1\{ i \leq (T \mod d)\}$ is the number of time steps for which $x_i = Xe_i$. Now note that as $T \geq d$,
  \begin{equation*}
    \sum_{i= 1}^d \sqrt{T_i} \geq \frac{1}{2}\sqrt{dT} \, . 
  \end{equation*}
  Indeed, let us check this by case disjunction. If $T \mod d \geq T / 2$, then the sum is at least $d / 2 \sqrt{T / d + 1} \geq \sqrt{dT} / T$. 
  Otherwise, $T \mod d < T / 2$, then $\floor T / d \rfloor \geq T/d - 1/2 \geq T / (2d)$, using the assumption that $T \geq d$. In this case, the sum is at least $d \sqrt{\floor T / d \rfloor} \geq \sqrt{dT /2}$. 
  
  Then, after rescaling $U$ by $1/ \sqrt d$, and noting that $[-U, U]^d \subset \ball(0, \sqrt d U)$, for any $T \leq (d/8)(Y/(XU))^2$,
  \begin{equation*}
    \sup_{(y_t)\in [-Y, Y]^T}  \sup_{\theta \in \ball(0, U)} R_T(\theta) \geq
    \frac{\sqrt{2}}{8} UXY \sqrt{T} \,.  
  \end{equation*}
\end{proof}

\begin{proof}{\bfseries of Theorem~\ref{thm:sq_lower_boundIII} \; }
  Once again, we assume that $UX \leq Y$, as the general result follows from applying it with $U = Y / X$ when $UX \geq Y$.
  
  The result in the $T \leq d$ regime follows from a somewhat trivial construction. At time $t$, let $x_t = X e_t$. Given the action from the learner, set $y_t = -Y \sign{a_t}$, and define $u_t = (U/\sqrt{T}) \sign{a_t}$. Consier the comparator with coordinates $u_t$ for $t \leq T$ and $0$ otherwise. Then $u_t$ has norm less than $U$, and the the total regret of the learner against $ \theta $ is at least  $T UXY/ \sqrt{T}$.

  Indeed, at all times $t$, we have $(a_t - y_t)^2 \geq Y^2$ since, e.g., $a_t \leq 0$ when $y_t = Y$. Similarly $((\langle \theta, x_t \rangle - y_t)^2) = (Y - UX / \sqrt T)^2$ and 
  \begin{multline*}
    (a_t - y_t)^2 -  (\langle \theta, x_t \rangle - y_t)^2
    = (a_t - y_t)^2 - \big((UX/\sqrt{T}) \sign{a_t} - y_t\big)^2 \\
    \geq Y^2 - \big(Y - UX / \sqrt{T}\big)^2
    =  2 \frac{UXY}{\sqrt T} - \frac{(UX)^2}{T} \geq \frac{UXY}{\sqrt T} \, .
  \end{multline*}
  The final inequality holds as long as $UX \leq Y$ and $T \geq 1$. Summing over $t \leq T$ gives the result.
\end{proof}
\interproofspace
\section{Lower Bound for the Hinge Loss}\label{app:hinge}
The next result shows a lower bound for the hinge loss that matches the
upper bound by \citet{mhammedi2020lipschitz} in the regime where $UX
\leq 1$. We prove it by relating the hinge loss to linear losses and
then applying a result of \citet{mcmahanstreeter}.
\begin{theorem}\label{thm:hingelower}
In $1$-dimensional online classification with the hinge loss, consider an algorithm that guarantees $R_T(0) \leq \eps$ for any sequence of data. Then for any $U,X>0$ and for any $T_0\geq 0$, there exists a sequence $(x_{t},y_{t})$ in $[-X,X]\times [-Y,Y]$ such that for some $T \geq T_0$, 
\begin{equation*}
  \sup_{\theta \in [-U, U]} R_{T}(\theta) \geq 0.336 \, (UX \bmin 1)\sqrt{T\ln\bigg(\frac{(UX \bmin 1)\sqrt{T}}{\delta}\bigg)} \, .
\end{equation*}
\end{theorem}

\begin{proof}
  Let us assume $UX \leq 1$, as the general case can be derived by applying the result with $U = 1/X$. Consider the feature sequence $x_t = X$, and denote by $a_t \in \R$ the sequence of actions produced by the algorithm. Then for any $|u| \leq U \leq 1/X$, we have $|u x_t y_t| \leq 1$ for all $t$ and therefore
  \begin{equation*}
    \ell(a_t, y_t) - \ell( ux_t, y_t)
    \geq 1 - a_ty_t - (1 -  uXy_t) 
    = -y_t(a_t - uX) \, . 
  \end{equation*}
  This implies that the regret of any algorithm for the hinge loss is lower bounded by its regret for a sequence of linear losses $\theta \mapsto - y_t \theta$ against the  comparator $uX$; denote this regret by $\tilde R_T(uX)$. Then for any sequence of $y_t$, we have $R_T(u) \geq R_T(uX)$.
  
  This implies in particular that $\tilde R_T(0) \leq R_T(0)$, which is assumed to be less than $\eps$ for any data sequence: the assumptions of \citet[Theorem~7]{mcmahanstreeter} are satisfied by the algorithm. Therefore, for any $\tilde U > 0$ and $T_0 \geq 0$, there exists a sequence of $y_t$'s, and a comparator $u$ with norm $\tilde U$ such that for some $T\geq T_0$,
  \begin{equation*}
    R_T(u / X) \geq \tilde R_T(u) \geq 0.336 \, \tilde U \sqrt{T \log \biggl( \frac{\tilde U \sqrt T}{\eps} \biggr)} \, .
  \end{equation*}
  The claimed bound follows by reparameterising $u$ by $u / X$.
\end{proof}

\end{document}